\documentclass[a4paper,11pt]{article}
\usepackage[top=3.25cm, bottom=3.25cm, left=2.5cm, right=2.5cm]{geometry}

\usepackage[utf8]{inputenc} 
\usepackage[T1]{fontenc}    
\usepackage{url}            
\usepackage{amsfonts}       
\usepackage{nicefrac}       
\usepackage{microtype}      

\usepackage{times}

\usepackage[english]{babel}
\usepackage[babel]{csquotes}
\usepackage{guit}
\usepackage{mdframed}
\usepackage{morefloats}
\usepackage{etex}
\usepackage{sidecap}
\usepackage{lipsum}
\usepackage{amsmath,amssymb,amsthm} 
\usepackage{mathtools} 
\usepackage{braket} 
\usepackage{bm} 
\usepackage{empheq}        
\usepackage{booktabs}                      
\usepackage{tabularx}                       
\usepackage{graphicx}                      
\usepackage{subfig}                          
\usepackage{caption}                       
\usepackage{xcolor} 
\usepackage{subfig}        
\usepackage{ucs}
\usepackage[normalem]{ulem}
\usepackage{multirow}
\usepackage{algorithm}
\usepackage{algorithmic}
\usepackage{amsfonts}
\usepackage{multicol}
\usepackage{ulem}

\usepackage{enumitem}

\usepackage[colorlinks,hyperindex]{hyperref}

\usepackage[font=small,skip=0pt]{caption}

\usepackage{thmtools,thm-restate}

\DeclarePairedDelimiter{\norma}{\lVert}{\rVert}

\newcommand{\X}{{\cal X}}
\newcommand{\E}{{\cal E}}
\newcommand{\F}{{\cal F}}
\newcommand{\Y}{{\cal Y}}
\newcommand{\Z}{{\cal Z}}

\newcommand{\D}{{\mathfrak D}}

\newcommand{\B}{{\cal B}}
\renewcommand{\L}{{\cal L}}
\newcommand{\data}{{\bf z}}
\newcommand{\mdata}{{\bf Z}}
\newcommand{\datax}{{\bf x}}
\newcommand{\datay}{{\bf y}}
\newcommand{\R}{\mathbb{R}}

\newcommand{\tr}{\rm tr}
\newcommand{\ran}{\rm Ran}

\newcommand{\trans}{^{\scriptscriptstyle \top}}

\makeatletter
\renewenvironment{proof}[1][\proofname]{\par
  \pushQED{\qed}%
  \normalfont \topsep6\p@\@plus6\p@\relax
  \trivlist
  \item[\hskip\labelsep
        \bfseries
    #1\@addpunct{.}]\ignorespaces
}{%
  \popQED\endtrivlist\@endpefalse
}
\makeatother
\def\eop{$\rule{1.3ex}{1.3ex}$}
\renewcommand\qedsymbol\eop

\newcommand{\cR}{{\cal R}}

\newcommand{\Exp}{\mathbb{E}}
\newcommand{\State}{\STATE}

\newcommand{\Dla}{{\D_\lambda}}

\newcommand{\argmin}[1]{\underset{#1}{\textrm{argmin}}~}

\declaretheorem[name=Theorem,refname=Thm.]{theorem}
\declaretheorem[name=Theorem,refname=Thm.,sibling=theorem]{theoremshortref}
\declaretheorem[name=Lemma,sibling=theorem]{lemma}
\declaretheorem[name=Proposition,refname=Prop.,sibling=theorem]{proposition}
\declaretheorem[name=Remark]{remark}

\declaretheorem[name=Definition,refname=Def.,sibling=theorem]{definition}

\declaretheorem[name=Assumption,refname=Asm.]{assumption}

\title{\sffamily\huge\bf Incremental Learning-to-Learn with Statistical Guarantees
\vspace{1.0truecm}
}

\author{ ~Giulia Denevi$^{1,2}$ ~~~~~ Carlo Ciliberto$^{3}$ ~~~~~ Dimitris Stamos$^{3}$ ~~~~~ Massimiliano Pontil $^{1,3}$ \\ {\small \hspace*{-2.0em} ~~giulia.denevi@iit.it~~~~~~~~c.ciliberto@ucl.ac.uk ~~~~ d.stamos.12@ucl.ac.uk ~~~~ massimiliano.pontil@iit.it}}

\begin{document}

\maketitle

\begin{abstract}
\noindent In learning-to-learn the goal is to infer a learning algorithm that works well on a class of tasks sampled from an unknown meta distribution. In contrast to previous work on batch learning-to-learn, we consider a scenario where tasks are presented sequentially and the algorithm needs to adapt incrementally to improve its performance on future tasks. Key to this setting is for the algorithm to rapidly incorporate new observations into the model as they arrive, without keeping them in memory. We focus on the case where the underlying algorithm is Ridge Regression parameterized by a positive semidefinite matrix. We propose to learn this matrix by applying a stochastic strategy to minimize the 
empirical error incurred by Ridge Regression on future tasks sampled from the meta distribution.
We study the statistical properties of the proposed algorithm and prove non-asymptotic bounds on its excess transfer risk, that is, the generalization performance on new tasks from the same meta distribution. We compare our online learning-to-learn approach with a state of the art batch method, both theoretically and empirically.
\end{abstract}

\section{INTRODUCTION}
\footnotetext[1]{Computational Statistics and Machine Learning, Istituto Italiano di Tecnologia, 16163 Genova, Italy}\footnotetext[2]{Department of Mathematics, University of Genova, 16146 Genova, Italy}\footnotetext[3]{Department of Computer Science, University College London, WC1E 6BT, London, UK}

Learning-to-learn (LTL) or meta learning aims at finding an algorithm that is best suited to address a class of learning problems (tasks). These tasks are sampled from an unknown meta distribution and are only partially observed via a finite collection of training examples, see \cite{baxter2000model,maurer2005algorithmic,thrun2012learning} and references therein. This problem plays a large role in artificial intelligence in that it can improve the efficiency of learning from human supervision. In particular, substantial improvement over ``learning in isolation'' (also known as independent task learning) is to be expected when the sample size per task is small, a setting which naturally arises in many applications~\cite{camoriano2017incremental,rebuffi2017icarl,rohrbach2013transfer,wu2014morphable}. 

LTL is particularly appealing when considered from an online or incremental perspective. In this setting, which is sometimes referred to as lifelong learning (see, e.g. \cite{ruvolo2013ella}), the tasks are observed sequentially -- via corresponding sets of training examples -- from a common environment and we aim to improve the learning ability of the underlying algorithm on future yet-to-be-seen tasks from the same environment. Practical scenarios of lifelong learning are wide ranging, including computer vision \cite{rebuffi2017icarl}, robotics \cite{camoriano2017incremental}, user modelling and many more.

Although LTL is naturally suited for the incremental setting, surprisingly theoretical investigations are lacking. 
Previous studies, starting from the seminal paper by Baxter \cite{baxter2000model}, have almost exclusively considered the setting in which the tasks are given in one batch~\cite{maurer2009transfer,maurer2013sparse,maurer2016benefit,pentina2014pac}, that is, the meta algorithm processes multiple datasets from the environment jointly and only once as opposed to sequentially and indefinitely. The papers~\cite{balcan2015efficient,herbster2016mistake} present results in an online framework which applies to a finite number of tasks using different performance measures. 
Perhaps most related to our work is \cite{alquier2016regret}, where the authors consider a general PAC-Bayesian approach to lifelong learning based on the exponentially weighted aggregation procedure. 
Unfortunately, this approach is not efficient for large scale applications as it entails storing the entire sequence of datasets during the meta learning process. LTL also bears strong similarity to multitask learning (MTL) \cite{caruana1998multitask} and much work has been done on the theoretical study of both batch \cite{ando2005framework,maurer2013excess} and online \cite{cavallanti2010linear} multitask learning algorithms. However multitask learning aims to solve the different -- and perhaps less challenging -- problem of learning well on a prescribed set of tasks (the learned model is tested on the same tasks used during training) whereas LTL aims to extrapolate to new tasks.

The principal contribution of this paper is to propose an incremental approach to learning-to-learn and to analyse its statistical guarantees. This incremental approach is appealing in that it efficiently processes one dataset at the time, without the need to store previously encountered datasets.  We study in detail the case of linear representation learning, in which an underlying learning algorithm receives in input a sequence of datasets and incrementally updates the data representation so as to better learn future tasks. Following previous work on LTL \cite{baxter2000model,maurer2009transfer}, we measure the performance of the incremental meta algorithm by the {\em transfer risk}, namely the average error obtained by running the underlying algorithm with the learned representation, over tasks sampled from the meta distribution. 

Specifically, in this work we choose the underlying algorithm to be Ridge Regression parameterized by a positive semidefinite matrix. The incremental LTL approach we propose aims at optimizing the future empirical error incurred by Ridge Regression over a class of linear representations. For this purpose, we propose to apply Projected Stochastic Subgradient Algorithm (PSSA). We show that the objective function of the resulting meta algorithm is convex and we give a non-asymptotic convergence rate for the algorithm in high probability. 
A remarkable feature of our learning bound is that it is comparable 
to previous bounds for batch LTL. 
Our proof technique leverages previous work on learning-to-learn \cite{maurer2009transfer} with tools from online convex optimization, see \cite{cesa2004generalization,hazan2016introduction} and references therein. 

The paper is organized as follows. In Sec.~\ref{sec:problem-formulation}, we review the LTL problem and describe in detail the case of linear feature learning with Ridge Regression. In Sec.~\ref{sec:online-learning-to-learn}, we present our incremental meta algorithm for linear feature learning. Sec.~\ref{sec:theory} contains our bound on the excess transfer risk for the proposed algorithm and in Sec.~\ref{sec:comparison-batch} we compare the bound to a previous bound for the batch setting. In Sec.~\ref{sec:experiments}, we report preliminary numerical experiments for the proposed algorithm and, finally, Sec.~\ref{sec:conclusions} summarizes the paper and highlight directions of future research.


\section{PROBLEM FORMULATION}
\label{sec:problem-formulation}

In the standard independent task learning setting the goal is to learn a functional relation between an input space $\X$ and an output space $\Y$ from a finite number of 
training examples. More precisely, given a loss function $\ell:\Y\times\Y\to\R$ measuring prediction errors and given a distribution $\mu$ on the joint data space $\Z = \X\times\Y$, the goal is to find a function $f:\X\to\Y$ minimizing the {\em expected risk}
\begin{equation}\label{eq:exp-risk-single-task}
\cR_\mu(f) = \Exp_{z\sim\mu}~\ell(f,z)
\end{equation}
where, with some abuse of notation, for any $z = (x,y) \in \Z$ we denoted $\ell(f,z) = \ell(f(x),y)$. In most practical situations the underlying distribution is {\em unknown} and the learner is only provided with a finite set $Z = (z_i)_{i=1}^n\in \Z^n$ of observations independently sampled from $\mu$. The goal of a learning algorithm is therefore, given such a {\em training} dataset $Z$ to return a ``good'' estimator $A(Z) = f_Z$ whose expected risk is small and tends to the minimum of Eq.~\eqref{eq:exp-risk-single-task} as $n$ increases.

A well-established approach to tackle the learning problem is offered by {\em regularized empirical risk minimization}. This corresponds to the family of algorithms $A_\phi$ such that, for any $Z\in\Z^n$,
\begin{equation}\label{eq:erm-single-task}
A_{\phi}(Z) = \argmin{f\in\F_\phi} \cR_Z(f) + \lambda\|f\|_{\F_\phi}^2
\end{equation}
where $\phi:\X\to\F_\phi$ is a feature map, $\F_\phi$ is the Hilbert space of functions $f:\X\to\Y$ such that $f(x) = \langle f, \phi(x) \rangle$ 
for any $x\in\X$ and
\begin{equation}
\cR_Z(f) = \frac{1}{n}\sum_{i=1}^n \ell(f,z_i)
\end{equation}
denotes the {\em empirical risk} of function $f$ on the set $Z$.

\subsection{Linear Feature Learning}

In this work we will focus on the case that $\mathcal{Y} \subseteq \R$, $\X \subseteq \R^d$ and $\phi:\R^d\to\R^k$ is a {\em linear} feature map (also known as a representation), corresponding to the action $\phi(x) = \Phi x$ of a matrix $\Phi\in\R^{k \times d}$ on the input space. It is well known (see e.g. \cite{argyriou2008convex}) that, setting $D =\frac{1}{\lambda} \Phi\trans \Phi \in\R^{d \times d}$, any problem of the form of \eqref{eq:erm-single-task} can be equivalently formulated as 
\begin{equation}\label{eq:linear-stl}
A_D(Z) = \argmin{w\in{\ran}(D)} \cR_Z(w) + w \trans D^\dagger w
\end{equation}
where, with some abuse of notation, we denoted with $\cR_Z(w)$ the empirical risk of the linear function $x \mapsto w\trans x$, for any $x\in\X$. Here, $D^\dagger$ denotes the pseudoinverse of $D$, which is positive semidefinite (PSD) but 
not necessarily invertible; when it is not invertible the constraint requiring $w$ to be in the range ${\ran}(D)\subseteq\R^d$ of $D$ is needed to grant the equivalence with Eq.~\eqref{eq:erm-single-task}. 

Since for any linear feature map $\phi$ there exists a PSD matrix $D$ such that Eq.~\eqref{eq:erm-single-task} and Eq.~\eqref{eq:linear-stl} are equivalent, in the following we will refer to $D$ as the {\em representation} used by algorithm $A_D$.

\subsection{Learning to Learn \texorpdfstring{$D$}{D}}

A natural question is how to choose a good representation $D$ for a given family of related learning problems. In this work we consider the approach of {\em learning} it from data. In particular, following the seminal work of \cite{baxter2000model}, we consider a setting where we are provided with an increasing number of tasks and our goal is to find a joint representation $D$ such that the corresponding algorithm $A_D$ is suited to address all such learning problems.
The underlying assumption is that all tasks that we observe {\em share a common structure} that algorithm $A_D$ can leverage in order to achieve better prediction performance. 

More formally, we assume that the tasks we observe are independently sampled from a meta distribution $\rho$ on the set 
of probability measures on $\Z$. According to the literature on the topic (see e.g. \cite{baxter2000model,maurer2005algorithmic}) we refer to the meta distribution $\rho$ as the {\em environment} and we identify each task sampled from $\rho$ by its corresponding distribution $\mu$, from which we are provided with a {\em training} dataset $Z\sim\mu^n$ of $n$ points sampled independently from $\mu$. While it is possible to consider a more general setting, for simplicity in this work we study the case where for each task we sample the same number $n$ of training points. In line with the independent task learning setting, the goal of a ``learning to learn'' algorithm is therefore to find the best parameter $D$ minimizing the so-called {\em transfer risk}
\begin{equation}\label{eq:ltl-expected-risk}
\E(D) = \Exp_{\mu\sim\rho}\Exp_{Z\sim\mu^n} ~~ \cR_\mu\big(A_D(Z)\big)
\end{equation}
over a set $\D$ of candidate representations. 

The term $\E(D)$ is the expected risk that the corresponding algorithm $A_D$, when trained on the dataset $Z$, would incur {\em on average with respect to the distribution of tasks $\mu$ induced by $\rho$}. That is, to compute the transfer risk, we first draw a task $\mu \sim \rho$ and a corresponding $n$-sample $Z\in \mathcal{Z}^{n}$ from $\mu ^{n}$, we then apply the learning algorithm to obtain an estimator $A_D(Z)$ and finally we measure the risk of this estimator on the distribution $\mu$.

The problem of minimizing the transfer risk in Eq.~\eqref{eq:ltl-expected-risk} given a finite number $T$ of training datasets  $Z_1,\dots,Z_T$ sampled from the corresponding tasks $\mu_1,\dots,\mu_T$, has been subject of thorough analysis in the literature \cite{baxter2000model,maurer2005algorithmic,maurer2016benefit}. Most work has been focused on the so-called ``batch'' setting, where all such training datasets are provided at once. However, by its nature, LTL is an ongoing (possibly never ending) process, with training datasets observed a few at the time. In 
such a scenario the meta algorithm should allow for an evolving representation $D$, which improves over time as new datasets are observed.
In the following we propose a meta algorithm to learn $D$ {\em online} with respect to the tasks, allowing us 
to transfer past experience about the environment in an efficient manner, {\em without requiring the memorization of training data}, 
which could be prohibitive in large scale applications. We will study the theoretical guarantees of the proposed algorithm and compare it to its batch counterpart in terms of both statistical and empirical performance.

\subsection{Connection with Multitask Learning}
\label{sec:MTL}
LTL is strongly related to {\em multitask learning} (MTL) and in fact, as we will see later for
the algorithm in Eq.~\eqref{eq:linear-stl}, approaches developed for MTL can be used 
as inspiration to design algorithms for LTL. In multi-task learning a fixed number of tasks $\mu_1,\dots,\mu_T$ is provided up front and, given $T$ datasets $Z_1,\dots,Z_T$, each sampled from its corresponding distribution, the goal is to find a joint representation $D$ incurring
a small {\em average expected risk} $\frac{1}{T}\sum_{t=1}^T \cR_{\mu_t}(A_D(Z_t))$. In this sense, the main difference between LTL and MTL is that the former aims to guarantee good prediction performance on {\em future tasks}, while the latter goal is to guarantee good prediction performance on the same tasks used to train $D$.

A well-established approach to MTL is {\em multitask feature learning} \cite{argyriou2008convex}. This method consists in solving the optimization problem

\begin{equation}
\label{eq:first-empirical-mtl-risk}
\min_{D\in\Dla} ~\frac{1}{T} \sum_{t=1}^T \min_{w\in{\ran}(D)} \cR_{Z_t}(w_t) + w_t\trans D^\dagger w_t
\end{equation}
over the set 
\begin{equation}\label{eq:D-space}
	\Dla = \big\{D ~ | ~ D\succeq 0,~ \tr(D)\leq 1/\lambda\big\}
\end{equation}

\noindent where $D\succeq0$ denotes the set of PSD matrices, ${\tr}(D)$ is the trace of $D$ and $\lambda$ is a positive parameter which controls the degree of regularization. This choice for $\Dla$ is motivated by the following variational form (see e.g. \cite[Prop.~4.2]{argyriou2008convex}) of the square trace norm of $W = [w_1,\dots, w_T]\in\R^{d \times T}$ 
\begin{equation}
\|W\|_{1}^2 =\frac{1} {\lambda} \inf_{D\in\textrm{Int}(\Dla)} ~ \sum_{t=1}^T w_t \trans D^{-1} w_t
\end{equation}
where $\textrm{Int}(\Dla)$ is the interior of $\Dla$, namely the set of PSD invertible matrices with trace strictly smaller than $1/\lambda$. This leads to the equivalent problem
\begin{equation}\label{eq:mtl-trace-norm}
\min_{W\in\R^{d \times T}} ~\frac{1}{T} \sum_{t=1}^T \cR_{Z_t}(w_t) + \gamma \|W\|_{1}^2
\end{equation}
with $\gamma = \lambda/T$. The trace norm of a matrix is defined as the sum ($\ell_1$-norm) of its singular values, and it is known to induce low-rank solutions for Problem~\eqref{eq:mtl-trace-norm}. Intuitively, this means that tasks are encouraged to {\em share a common set of features (or representation)}. In this paper, we adopt this perspective to design our approach for online linear feature learning. 


\section{ONLINE LEARNING-TO-LEARN}
\label{sec:online-learning-to-learn}

Motivated by the above connection with multitask learning, we propose an online LTL approach to approximate the solution of the learning problem
\begin{equation}
\min_{D\in\Dla}~ \E(D)
\end{equation}
over the set $\Dla$ introduced in Eq.~\eqref{eq:D-space}. We consider the setting in which we are provided with a stream of independent datasets $Z_1,\dots,Z_T,\dots$, each sampled from an individual task distribution $\mu_1,\dots,\mu_T,\dots$ 
and our goal is to find an estimator in $\Dla$ that improves {\em incrementally} as the number of observed tasks $T$ increases.

\subsection{Minimizing the Future Empirical Risk}\label{sec:ltl-empirical-risk}

A key observation motivating the online procedure proposed in this work, is that in the {\em independent task learning} setting, standard results from learning theory (see e.g. \cite{shalev2014understanding}) allow one to control the statistical performance of regularized empirical risk minimization, providing bounds on the {\em generalization error} of $A_D$ as
\begin{equation}\label{eq:stl-generalization-error}
\hspace{-.05truecm}	\Exp_{\scalebox{0.75}{$Z\sim\mu^n$} } |\cR_\mu\big(A_D(Z)\big) {-} \cR_Z\big(A_D(Z)\big)| \leq G(D,n)
\end{equation}
where $G(\cdot,n)$ is a decreasing function converging to $0$ as $n\to+\infty$, while $G(D,\cdot)$ is a measure of complexity of $D$, which is large for more ``expressive'' representations and smaller otherwise. 

Eq.~\eqref{eq:stl-generalization-error} suggests us to use the empirical risk $\cR_Z$ as a proxy for the expected risk $\cR_\mu$. Therefore, we introduce the so-called {\em future empirical risk} \cite{maurer2009transfer,maurer2016benefit},
\begin{equation}\label{eq:empirical-transfer-risk}
\hat \E(D) = \Exp_{\mu\sim\rho}\Exp_{Z\sim\mu^n}~ \cR_Z\big(A_D(Z)\big)
\end{equation}
and consider the related problem
\begin{equation}\label{eq:ltl-empirical-minimization}
	\min_{D\in\Dla} ~ \hat\E(D),
\end{equation}
which in the sequel, introducing the shorthand notation $\L_Z(D) = \cR_Z(A_D(Z))$ for any PSD matrix $D$, will be rewritten as
\begin{equation}\label{eq:expected-stochastic-optimization}
\min_{D\in\Dla} ~ \Exp_{\mu\sim\rho}\Exp_{Z\sim\mu^n} ~ \L_Z(D)
\end{equation}
to highlight the dependency on $Z$.

Problem \eqref{eq:expected-stochastic-optimization} can be approached with stochastic optimization strategies. Such methods proceed by sequentially sampling a point (dataset in this case) $Z$ and performing an update step. 
In recent years, stochastic optimization, finding its origin in the Stochastic Approximation method by Robbins and Monro \cite{robbins1951stochastic}, has been effectively used to deal with large scale applications. 
We refer to \cite{nemirovski2009robust} 
for a more comprehensive discussion about this topic. We therefore propose to apply Projected Stochastic Subgradient Algorithm (PSSA) \cite{shamir2013stochastic}, to solve the optimization problem in Eq.~\eqref{eq:expected-stochastic-optimization}. The candidate representation coincides in this case with the mean after $T$ iterations $\bar{D}_T$ and it is known as Polyak-Ruppert averaging scheme 
~\cite{nemirovskii1983problem,polyak1992acceleration} in the optimization literature. Algorithm~\ref{alg:general-pssa} reports the application of PSSA to $\hat \E$ when $\L_Z$ is convex on the set of PSD matrices. It requires iteratively: $i)$ sampling a dataset $Z$, $ii)$ performing a step in the direction of a subgradient of $\L_Z$ at the current point, and $iii)$ projecting onto the set $\Dla$ (which can be done in a finite number of iterations, see \autoref{lemma_proj} in Appendix 
\ref{app:E} 
). Note that in this case, since the function $\L_Z$ is convex, there is no ambiguity in the definition of the subdifferential $\partial \L_Z$ (see e.g. \cite{bertsekas2003convex}) and we can rely on the convergence of  Algorithm~\ref{alg:general-pssa} to a global minimum of $\hat \E$ over $\Dla$ for a suitable choice of step-sizes, as discussed in Sec.~\ref{sec:theory}.

\begin{algorithm}[t]
\caption{PSSA applied to $\hat{\mathcal{E}}$}\label{alg:general-pssa}
\begin{algorithmic}
   \State ~
   \State {\bfseries Input:} $T$ number of tasks, $\lambda>0$ hyperparameter, $\{\gamma_t \}_{t\in\mathbb{N}}$ step sizes. 
    \State ~
   \State {\bfseries Initialization:} $D^{(1)} \in \Dla$ $\bigl($ e.g. $D^{(1)} = \frac{1}{\lambda} I$ $\bigr)$ 
   \State ~
   \State {\bfseries For} $t=1$ to $T$:
   \State \qquad Sample ~~ $\mu_t\sim\rho$, $Z_t\sim\mu_t^n$.
   \State \qquad Choose ~~$U_t \in \partial \L_{Z_t}(D^{(t)})$
   \State \qquad Update ~~$D^{(t+1)} = {\rm proj}_{\Dla} 
(D^{(t)} - \gamma_t U_t)$  
	\State ~
 \State {\bfseries Return $\displaystyle \bar{D}_T = 
\frac{1}{T} \sum_{t = 1}^T D^{(t)}$} 
\end{algorithmic}
\end{algorithm}

\subsection{LTL with Ridge Regression}

In this section, we focus on the case that the loss function $\ell:\Y\times\Y\to\R$ corresponds to least-squares, namely $\ell(y,y') = (y - y')^2$ for any $y,y'\in\Y \subseteq \R$. In this setting, given a dataset $Z\in\Z^n$, algorithm $A_D$ is equivalent to perform the following variant to {\em Ridge Regression} 
\begin{equation}\label{eq:ridge-regression-with-D}
\min_{w\in{\ran}(D)} ~ \frac{1}{n} \|\datay - Xw\|^2 +  w\trans D^\dagger w
\end{equation}
where $X\in\R^{n \times d}$ is the matrix with rows corresponding to the input points $x_i\in\mathbb{R}^d$ in the dataset $Z$ and $\datay\in\R^n$ the vector with entries equal to the corresponding output points $y_i\in\R$. The solution to Eq.~\eqref{eq:ridge-regression-with-D} can be obtained in closed form, in particular (see e.g. \cite{argyriou2008convex,maurer2009transfer})
\begin{equation}\label{eq:ridge-regression-closed-form}
A_D(Z) = DX\trans\big(XDX\trans + nI\big)^{-1}\datay.
\end{equation}
Plugging this solution in the definition of $\L_Z(D)$, a direct computation yields that
\begin{equation}\label{eq:ridge-regression-obj-funct}
\L_Z(D) = n \big\|(XDX\trans + nI)^{-1}\datay\big\|~^2.
\end{equation}
The following result characterizes some key properties of the function $\L_Z$ in Eq.~\eqref{eq:ridge-regression-obj-funct}, which will be useful in our subsequent analysis. We denote by $\mathcal{B}_r\subseteq\R^d$ the ball of radius $r>0$ centered at $0$.

\begin{proposition}[Properties of $\L_Z$ for the Square Loss]\label{prop:least-squares}
Let $\X \subseteq \B_1$, $\Y \subseteq [0,1]$ and $\ell$ be the square loss. Then, for any dataset $Z\in\Z^n$ the following properties hold:
\noindent
\begin{enumerate}[topsep=0pt]
\setlength{\itemsep}{0em}
\item $\L_Z$ is convex on the set of PSD matrices.
\setlength{\itemsep}{0em}
\item $\L_Z$ is $\mathcal C^\infty$ and,~for every PSD $D\in\R^{d\times d}$, 
{
\begin{equation}\label{ltl-least-squares-gradient}
\nabla\L_Z(D) = -n X\trans M(D)^{-1}S(D)M(D)^{-1} X\\
\end{equation}
where
\begin{align}
M(D) & = XDX\trans + nI\\
S(D) & = \datay\datay \trans M(D)^{-1} + M(D)^{-1}\datay\datay\trans.
\end{align}}
\setlength{\itemsep}{0em}
\item $\L_Z$ is $2$-Lipschitz w.r.t. the Frobenius norm.
\setlength{\itemsep}{0em}
\item $\nabla\L_Z$ is {$6$-Lipschitz} w.r.t. the Frobenius norm.
\setlength{\itemsep}{0em}
\item $\L_Z(D)\in[0,1]$,~ for any PSD matrix $D\in\R^{d\times d}$.
\end{enumerate}
\end{proposition}

The proposition above establishes the convexity of Problem \eqref{eq:ltl-empirical-minimization} for the case of the square loss. This fact is important in that it guarantees no ambiguity in applying Algorithm~\ref{alg:general-pssa} to our setting and moreover, since $\L_Z$ is differentiable, Algorithm~\ref{alg:general-pssa} becomes a {\em Projected Stochastic Gradient Algorithm}.



\section{THEORETICAL ANALYSIS}\label{sec:theory}

In this section, we study the statistical properties of Algorithm~\ref{alg:general-pssa} for the case of the square loss. Below we report the main result of this work, which characterizes the non-asymptotic behavior of the estimator $\bar{D}_T$ produced by Algorithm~\ref{alg:general-pssa} with respect to a minimizer $D_* \in \textrm{argmin}_{D\in\Dla}\E(D)$. 
To present our results we introduce the $d \times d$ matrix
\begin{equation}\label{eq:covariance}
C_{\rho} = \Exp_{\mu\sim\rho}\Exp_{(x,y)\sim\mu} \hspace{.01truecm} [ xx\trans]
\end{equation}
denoting the covariance of the input data, obtained by averaging over all input marginals sampled from $\rho$. This quantity plays a key role in identifying the environments $\rho$ for which it is favorable to adopt a LTL strategy instead of solving the tasks independently, see  \cite{maurer2009transfer} for a discussion. We also denote with $\|C_{\rho}\|_\infty$ the operator norm of $C_{\rho}$, which corresponds to the largest singular value.

\begin{theoremshortref}[Online LTL Bound]\label{thm:main-ltl-online}
Let $\mathcal{X} \subseteq \mathcal{B}_1$, $\mathcal{Y} \subseteq [0,1]$ and $\ell$ be the square loss. Let $\bar{D}_T$ be the output of Algorithm~\ref{alg:general-pssa} with step sizes $\gamma_t = (\lambda\sqrt{2t})^{-1}$. Then, for any $\delta\in(0,1]$
\begin{equation}
\mathcal{E}(\bar{D}_T) - \mathcal{E}(D_{*}) 
\leq  \frac{4 \sqrt{2 \pi}  \|C_{\rho}\|_\infty^{1/2} }{\sqrt{n}}
\frac{1 + \sqrt{\lambda}}{\lambda} \\
~~ + \frac{4 \sqrt{2}}{\lambda \sqrt{T}} + 
\sqrt{\frac{8 \log \bigl( 2/\delta \bigr)}{T}}
\end{equation}
with probability at least $1-\delta$ with respect to the independent sampling of tasks $\mu_t\sim\rho$ and training sets $Z_t\sim\mu_t^n$~for any $t \in \{1,\dots,T\}$. 
\end{theoremshortref}

In Sec.~\ref{sec:comparison-batch}, we will compare \autoref{thm:main-ltl-online} with the statistical bounds available for state of the art LTL batch  procedures. We will see that the statistical behaviour of these two approaches is essentially equivalent, with Online LTL being more appealing given the lower requirements in terms of both number of computations and memory.
 
In the rest of this section we give a sketch of the proof for \autoref{thm:main-ltl-online}. Proofs of intermediate results are reported in the appendix.

\subsection{Error Decomposition}

The statistical analysis of Algorithm~\ref{alg:general-pssa} hinges upon the following decomposition for the excess transfer risk of the estimator $\bar{D}_T$:
\begin{align}\label{eq:ltl-error-decomposition}
\E(\bar{D}_T)   - \E(D_*) & = \E(\bar{D}_T) \pm \hat \E(\bar{D}_T) \pm \hat \E(D_*) - \E(D_*) \nonumber \\
& \leq 2\sup_{D\in\Dla} |\E(D)-\hat \E(D)| + \hat \E(\bar{D}_T) - \hat \E(D_*) \nonumber \\
& \leq \underbrace{ 2\sup_{D\in\Dla} |\E(D)-\hat \E(D)|}_{\substack{\textrm{Uniform generalization} \\ \textrm{error} } } + \underbrace{\vphantom{2\sup_{D\in\Dla} |\E(D)-\hat \E(D)|}\hat\E(\bar{D}_T) - \hat\E(\hat D_*)}_{\substack{\textrm {Excess future} \\ \textrm{empirical risk} } } 
\end{align}
where the matrix $\hat D_*$ denotes a minimizer of the future empirical risk over $\Dla$,
that is, $\hat D_* \in \textrm{argmin}_{D\in\Dla} \hat\E(D)$.

Eq.~\eqref{eq:ltl-error-decomposition} decomposes $\E(\bar{D}_T) - \E(D_*)$ in a {\em uniform generalization error}, implicitly encoding the complexity of the class of algorithms parameterized by $D$ and an {\em excess future empirical risk}, measuring the discrepancy between the estimator $\bar{D}_T$ and the minimizer $\hat D_*$ of $\hat \E$. In the following we describe how to bound these two terms.

\subsection{Bounding the Uniform Generalization Error}

Results providing generalization bounds for the class of regularized empirical risk minimization algorithms $A_D$ considered in this work are well known. The following result, which is taken from \cite{maurer2009transfer}, leverages an explicit estimate of the generalization bound $G(D,n)$ introduced in Sec.~\ref{sec:ltl-empirical-risk} for independent task learning (see Eq.~\eqref{eq:stl-generalization-error}) to obtain a uniform bound over the class of algorithms parametrized by $\Dla$.

\begin{restatable}[Uniform Generalization Error Bound for Algorithm~\ref{alg:general-pssa}]{proposition}{uniformboundfirst}
\label{prop:uniform-generalization-bound}
Let $\mathcal{X} \subseteq \mathcal{B}_1$, $\Y\subseteq[0,1]$ {and let $\ell$ be the square loss}, then
\begin{equation}
\sup_{D\in\Dla}|\E(D) - \hat \E(D)|\leq \frac{2\sqrt{2\pi}\|C_{\rho}\|_\infty^{1/2}}{\sqrt{n}}\frac{1 + \sqrt{\lambda}}{\lambda}.
\end{equation}
\end{restatable}
For completeness, we report the proof of this proposition 
in Appendix \ref{app:B3}.

\subsection{Bounding the Excess Future Empirical Risk}

Providing bounds for the excess future empirical risk introduced in Eq.~\eqref{eq:ltl-error-decomposition} consists in studying the convergence rates of Algorithm~\ref{alg:general-pssa} to the minimum of $\hat \E$ over $\Dla$ {\em in high probability with respect to the sample of~$T$ tasks $\mu_t$ from $\rho$ and datasets $Z_t$ from $\mu_t^n$ for any $t\in \{1,\dots,T\}$}. 

To this end, we leverage classical results from the online learning literature \cite{hazan2016introduction}.  In online learning, the performance of 
an online algorithm returning a sequence $\{D^{(t)} \}_{t = 1}^T$ over $T$ trials is measured in terms of its {\em regret}, which in the context of this work corresponds to  
\begin{equation}\label{eq:regret-definition}
R_T = \frac{1}{T}\sum_{t=1}^T \L_{Z_t}(D^{(t)}) - \min_{D\in\Dla}\frac{1}{T}\sum_{t=1}^T \L_{Z_t} (D).
\end{equation}
Differently from the statistical setting considered in this work, in the online setting no 
assumption is made about the data generation process of $Z_1,\dots,Z_T$, which could be even adversely generated. Therefore, an algorithm that is able to solve the online problem (i.e. if its regret vanishes as $T \to \infty$) can be also expected to solve the corresponding problem in the statistical setting. This is indeed the case for Algorithm~\ref{alg:general-pssa}, for which the following lemma provides a non-asymptotic regret bound.

\begin{restatable}[Regret Bound for Algorithm~\ref{alg:general-pssa}]{lemma}{regretboundour}
\label{lem:regret-bound}
Let $\mathcal{X} \subseteq \mathcal{B}_1$, $\mathcal{Y} \subseteq [0,1]$ and $\ell$ be the square loss. 
Then the regret of Algorithm~\ref{alg:general-pssa} with step-sizes $\gamma_t = (\lambda \sqrt{2t})^{-1}$ 
is such that
\begin{equation}
R_T\leq  \frac{4 \sqrt{2}}{\lambda\sqrt{T}}.
\end{equation}
\end{restatable}
This lemma is a corollary of \autoref{prop:least-squares} combined with classical results on regret bounds for Projected Online Subgradient Algorithm \cite{hazan2016introduction}. We refer the reader to 
Appendix \ref{app:D1} for a more in-depth discussion and for a detailed proof.

In our setting, the datasets $Z_1,\dots,Z_T$ are assumed to be independently sampled from the underlying environment. Combining this assumption with the regret bound in \autoref{lem:regret-bound}, we can control the excess future empirical risk by means of so-called {\em online-to-batch conversion} results \cite{cesa2004generalization,hazan2016introduction}, leading to the following proposition.

\begin{restatable}[Excess Future Empirical Risk Bound for Algorithm~\ref{alg:general-pssa}]{proposition}{sgdrateour}
\label{prop:bound-empirical-transfer-error}
Let $\mathcal{X} \subseteq \mathcal{B}_1$, $\Y\subseteq[0,1]$ and let $\ell$ be the square loss. Let $\mu_1,\dots,\mu_T$ be independently sampled from $\rho$ and $Z_t$ sampled from $\mu_t^n$ for $t\in \{1,\dots,T\}$. Let $\bar{D}_T$ be the output of Algorithm~\ref{alg:general-pssa} with step sizes $\gamma_t = (\lambda\sqrt{2t})^{-1}$. Then, for any $\delta\in(0,1]$
\begin{equation}\label{eq:bound-empirical-error}
\hat \E(\bar{D}_T) - \hat \E(\hat D_*) \leq 
\frac{4 \sqrt{2}}{\lambda \sqrt{T}}
+ \sqrt{\frac{8\log(2/\delta)}{T}}
\end{equation}
with probability at least $1-\delta$. 
\end{restatable}

The result above follows by combining \autoref{prop:least-squares} with online-to-batch results (see e.g. Thm.~$9.3$ in \cite{hazan2016introduction} and \cite{cesa2004generalization}). In 
Appendix \ref{app:D2} we provide the complete proof of this statement together with
a more detailed discussion about this topic. At this point we are ready to give the proof of \autoref{thm:main-ltl-online}.

\begin{proof}[{\bf Proof of} \bf\autoref{thm:main-ltl-online}.] The claim follows by combining \autoref{prop:uniform-generalization-bound} and \autoref{prop:bound-empirical-transfer-error} in the decomposition of the error $\E(\bar{D}_T) - \E(D_*)$ given in Eq.~\eqref{eq:ltl-error-decomposition}.  
\end{proof}


\section{ONLINE LTL VERSUS BATCH LTL}\label{sec:comparison-batch}

In this section, we compare the statistical guarantees obtained for our online meta algorithm 
with a state of the art batch LTL method for linear feature learning. We also comment on the computational cost of both procedures.

\subsection{Statistical Comparison}

Given a finite collection $\mdata = \{Z_1,\dots,Z_T\}$ of datasets,
a standard approach to approximate a minimizer of the future empirical risk $\hat \E$ is to take a
representation $\hat D_T$ minimizing the multitask empirical risk
\begin{equation}
\hat \E_\mdata(D) =  \frac{1}{T} \sum_{t = 1}^T \cR_{Z_t}(A_D(Z_t))
\label{eq:empirical-mtl-risk}
\end{equation}
over the set $\Dla$. Such a choice has been extensively studied in the LTL literature \cite{baxter2000model,maurer2009transfer,maurer2013sparse,maurer2016benefit}. Here we report a result analogous to \autoref{thm:main-ltl-online}, characterizing 
the discrepancy between the transfer risks of $\hat D_T$ and $D_*$.

\begin{restatable}[Batch LTL Bound]{theoremshortref}{batchtheorerecall}
\label{thm:main-batch}
Let $\mathcal{X} \subseteq \mathcal{B}_1$, $\Y\subseteq[0,1]$ and let $\ell$ be the square loss. Let tasks $\mu_1,\dots,\mu_T$ be independently sampled from $\rho$ and $Z_t$ sampled from $\mu_t^n$ for $t\in\{1,\dots,T\}$. Let $\hat D_T$ be a minimizer of the multitask empirical risk in Eq.~\eqref{eq:empirical-mtl-risk} {over the set $\Dla$}. Then, for any $\delta\in(0,1]$

\begin{equation}
\mathcal{E}(\hat D_T) - \mathcal{E}(D_*)
\le \frac{4 \sqrt{2 \pi} \norma{C_\rho}_\infty ^{1/2}}{\sqrt{n}}
\frac{1 + \sqrt{\lambda}}{\lambda} \\
+ \frac{2 \sqrt{2 \pi}}{\lambda \sqrt{T}} 
+  \sqrt{\frac{2\log \bigl( 2/\delta \bigr)}{T}}
\end{equation}
with probability at least $1-\delta$. 
\end{restatable}

The result above is obtained by further decomposing the error $\E(\hat{D}_T)-\E(D_*)$ as done in Eq.~\eqref{eq:ltl-error-decomposition}. In particular, since the multitask empirical error provides an estimate for the future empirical risk, it is possible to to control the overall error by further bounding the term $| \hat \E(D) - \hat\E_\mdata(D)|$ uniformly with respect to $D\in\Dla$. This last result was originally presented in \cite{maurer2009transfer};
in Appendix \ref{Appendix C} 
we report the complete analysis of such decomposition, leading to the bound in \autoref{thm:main-batch}.

\subsection{Statistical Considerations}
\label{sec:5.2}
We now are ready to compare the bounds on the excess transfer risk
for the representations resulting from the application of the online procedure 
(see \autoref{thm:main-ltl-online}) and the batch one (see \autoref{thm:main-batch});
this provides a first indication of the behavior of the different algorithms. However, it should be kept in mind that we are comparing upper bounds, hence our considerations are not conclusive and further analysis by means of lower bounds for both algorithms would be valuable.

At first, we observe that both bounds reflect the fact that the LTL setting is more 
challenging than the MTL setting. As we have remarked in Sec. \ref{sec:MTL}, MTL aims at bounding the excess averaged expected risk relative to a {\em fixed} set of $T$ tasks. 
Such bounds can be expressed as a function $B(n,T)$, with $B(n, T) \to 0$
as $n \to \infty$ for {\em any} $T$. 
In LTL, the bounds study the behavior of the excess transfer
risk, measuring how well the candidate estimator will work 
on a {\em new} task sampled from the environment. Therefore,  
it is no longer possible that $B(T, n) \to 0$
for $n \to \infty$ for any $T$, we only have that $B(T, n) \to 0$ when $n$ and $T$ simultaneously tend to infinity.
 
\autoref{thm:main-ltl-online} and \autoref{thm:main-batch} are both composed of three terms. The first term is exactly the same for both procedures and this is obvious looking
at the decompositions used to deduce both results. This term can be interpreted as
a within-task-estimation error, that depends on the number of points 
$n$ used to train the underlying learning algorithm (in our case Ridge Regression with a linear 
feature map). This term, similarly to the MTL setting, highlights the advantage of exploiting the relatedness of the tasks in the learning process in comparison to independent task learning (ITL). Indeed, if the inputs are distributed on a high dimensional manifold, then $\|C_\rho\|_\infty \ll 1$, while upper bounds for ITL have a leading constant of 1. In particular, $\|C_\rho\|_\infty = 1/d$ if the marginal distributions of the tasks are uniform on the $d-1$ dimensional unit sphere; see also \cite{maurer2009transfer,maurer2016benefit} for a more detailed discussion. The last term in the bounds expresses the dependency on the confidence parameter $\delta$ and it is again approximately the same for the batch and the online case.  It follows that the main role in the comparison between the online and batch bounds is driven by the
middle term, which expresses the dependency of the bound on the number of tasks $T$. This term originates in 
different ways: in the batch approach it is derived from the application of uniform bounds
and it can be interpreted as an inter-task estimation error, while in the online approach,
it plays the role of an optimization error. Despite the different derivations, 
we can ascertain from the explicit formula of the bounds that this term is approximately the same for both procedures. This is remarkable since it implies that the representation resulting from our online procedure enjoys the same statistical guarantees than the batch one, despite its more parsimonious memory and computational requirements. 

\subsection{Computational Considerations}\label{sec:comparison-batch-computations}

After discussing the theoretical comparison between the online and the batch LTL approach, in this section we point out some key aspects regarding the computational costs of both procedures.
\vspace{.2truecm}

\noindent {\bf Memory}. The batch LTL estimator corresponds to the minimizer of the empirical risk in Eq.~\eqref{eq:empirical-mtl-risk} over {\em all tasks observed so far}. The corresponding approach would therefore require storing in memory all training datasets as they arrive in order to perform the optimization. This is clearly not sustainable in the incremental setting, since tasks are observed sequentially and, possibly indefinitely, which would inevitably lead to a memory overflow. On the contrary, in line with most stochastic methods, online LTL has a small memory footprint, since it requires to store only one dataset at the time, allowing to ``forget'' it as soon as one gradient step is performed.
\vspace{.2truecm}

\noindent {\bf Time}. Online LTL is also advantageous in terms of the number of iterations performed whenever a new task is observed. Indeed, for every new task, online LTL performs {\em only one step} of gradient descent for a total of $T$ steps after $T$ tasks. On the contrary, batch LTL requires finding a minimizer for Eq.~\eqref{eq:empirical-mtl-risk}, which cannot be obtained in closed form but requires adopting an iterative method such as Projected Gradient Descent (see e.g. \cite{combettes2005signal}). These methods typically require $k$ iterations to achieve an error of the order of $O(1/k)$ from the optimum (better rates are possible adopting accelerated methods). However, since for any new task batch LTL needs to find a minimizer for the multitask empirical error in Eq.~\eqref{eq:empirical-mtl-risk} from scratch, this leads to a total of $Tk$ iterations after $T$ tasks. Noting that every such iteration requires to compute $T$ gradients of $\L_Z$ in contrast to the single one of PSSA, this shows that online LTL requires much less operations. {In the batch case,} a ``warm-restart'' strategy can be adopted to initialize the Projected Gradient Descent with the representation learned during the previous step, however, as we empirically observed in Sec.~\ref{sec:experiments}, online LTL is still significantly faster than batch.


\section{EXPERIMENTS}\label{sec:experiments}

In this section, we report preliminary empirical evaluations of the online LTL strategy proposed in this work. In particular we compare our method with its batch (or offline) counterpart \cite{maurer2009transfer} and independent task learning (ITL) with standard Ridge Regression.

In all experiments, we obtain the online and batch estimators $\bar D_{\lambda,T_{\rm tr}}$ and $\hat D_{\lambda,T_{\rm tr}}$ by learning them on a dataset $\mdata_{\rm tr}$ of $T_{\rm tr}$ {\em training} tasks, each comprising $n$ input-output pairs $(x,y)\in\X\times\Y$. Below to simplify our notation we omit the subscript $T_{\rm tr}$ in these estimators. We perform this training for different values of $\lambda\in \{\lambda_1,\dots,\lambda_m\}$ and select the best estimator based on the prediction error measured on a separate set $\mdata_{\rm va}$ of $T_{\rm va}$ {\em validation} tasks. Once such optimal $\lambda$ value has been selected, we report the generalization performance of the corresponding estimator on a set $\mdata_{\rm te}$ of $T_{\rm te}$ {\em test} tasks. Note that the tasks in the test and validation sets $\mdata_{\rm te}$ and $\mdata_{\rm va}$ are all provided with both a training and test datasets $Z,Z'\in\Z^n$. Indeed, in order to evaluate the performance of a representation $D$, we need to first train the corresponding algorithm $A_D$ on $Z$, and then test its performance on $Z'$~(sampled from the same distribution), by  computing the empirical risk $\cR_{Z'}(A_D(Z))$. For all methods considered in this setting, we perform parameter selection over $30$ candidate values of $\lambda$ over the range $[10^{-6},10^3]$ with logarithmic spacing.

\begin{figure}[t!]
\centering
\includegraphics[width=0.58\columnwidth]{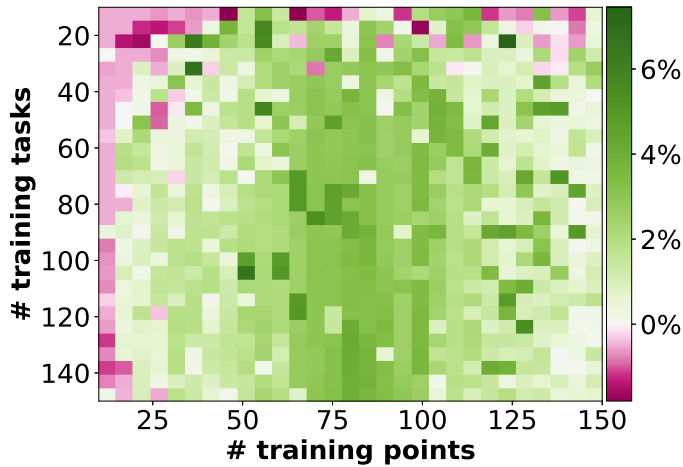}
\caption{Relative improvement (in $\%$) of our online LTL algorithm over the ITL baseline for a varying range of training tasks and number of samples per task.\label{fig:grid}
}
\end{figure}

In the online setting the training datasets arrive incrementally and a few at the time. Therefore model selection is performed {\em in parallel}: the system keeps track of all candidate representation matrices $\bar D_{\lambda_1},\dots,\bar D_{\lambda_m}$ and whenever a new training task is presented, these matrices are all updated by incorporating the corresponding new observations. The best representation is then returned at each iteration, 
based on its performance on the validation set $\mdata_{\rm va}$.
\vspace{.2truecm}

\noindent {\bf Synthetic Data}. We considered a regression problem on $\X \subseteq \R^d$ with $d = 50$ and a variable number of training tasks $T_{\rm tr} \in \{ 10,\dots 150\}$ each comprising $n \in \{10, \dots, 150\}$ training points. We also generated $T_{\rm te} = 100$ test tasks and we sampled a number $T_{\rm va}$ of validation tasks equal to $25\%$ of $T_{\rm tr}$. For each task, the corresponding dataset $(x_i,y_i)_{i=1}^n$ was generated according to the linear regression equation $y = w\trans x + \epsilon$, with $x$ sampled uniformly on the unit sphere in $\mathbb{R}^d$ and $\epsilon\sim\mathcal{N}(0,0.2)$. Analogously to the input points, the tasks predictors $w\in\R^d$ were generated as $P\tilde w$ with the components of $\tilde w\in\R^{d/2}$ sampled from $\mathcal{N}(0,1)$ and then $\tilde w$ normalized to have unit norm, and $P\in\R^{d\times d/2}$ a matrix with orthonormal rows. The tasks $w$ were generated according to this model to reflect the assumption of sharing a low dimensional representation, which needs to be inferred by the LTL algorithm.

\begin{figure}[t!]
\centering
\includegraphics[width=0.58\columnwidth]{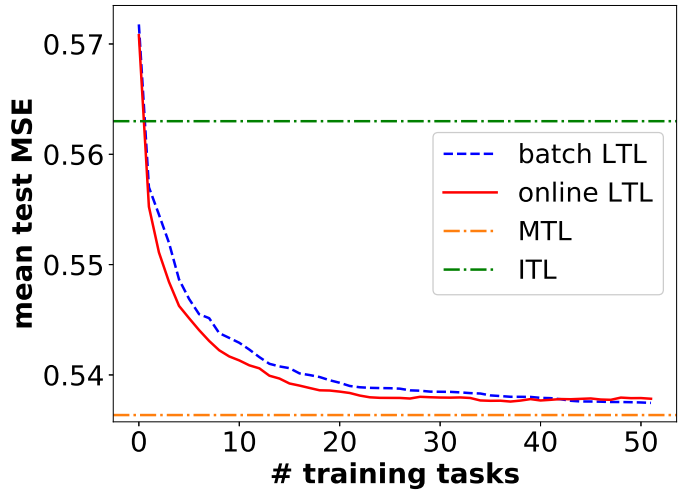}
\caption{Performance of online LTL, batch LTL, ITL and MTL (on the test set) during one single trial of online learning on the synthetic dataset as the number of training tasks increases incrementally.
\label{fig:one-run-synth}
}
\end{figure}
Figure~\ref{fig:grid} reports the comparison between the baseline ITL and the proposed online LTL approach in terms of the relative difference of the prediction error on test tasks for the two methods. More precisely, given the mean squared errors (MSE) $R_{\rm oLTL}$ of online LTL and $R_{\rm ITL}$ of ITL averaged across the test tasks, we report the ratio $(R_{\rm ITL} - R_{\rm oLTL})/R_{\rm ITL}$ as a percentage improvement. Results are reported across a range of $T_{\rm tr}$ and $n$. We note that the regime considered for these experiments is particularly favorable to LTL, consistently outperforming ITL, which does not leverage any shared structure. As noted in Sec. \ref{sec:5.2}, when the number of training points per task is small, the LTL algorithm is unable to capture the underlying representation, even if several tasks are provided in training. To provide further evidence of the performance of online LTL, Figure~\ref{fig:one-run-synth} compares the prediction error of online LTL, batch LTL, and ITL as the number of training tasks $T_{\rm tr}$ increases from $1$ to $50$ 
and the number of samples per task is fixed to $n=25$. In particular we considered the setting where the task datasets are provided incrementally one at the time and the different methods update their corresponding representation accordingly. We also report the performance of the multitask algorithm (MTL) described in Sec.~\ref{sec:MTL}, performing trace norm regularization 
{\em on the test set}. Clearly, MTL does not fit the learning to learn setting since it optimizes the representation on the test set. In this sense, loosely speaking, the MTL performance provides a lower bound on the performance that we can expect from an ``ideal'' LTL algorithm. Indeed, we note that MTL consistently outperforms all LTL methods which, however, tend to converge to it as more training tasks are provided.

\begin{table}[t]
    \small
    \caption{Computational times (in seconds) of online and batch LTL for a varying number $T_{\rm tr}$ of training tasks and $n$ of samples per task.}\label{tab:computational-performance}
    \begin{center}
        \begin{tabular}{rcc|cc|cc}
            \toprule
            ${\boldsymbol T_{\rm tr}}$  & \multicolumn{2}{c|}{50} & \multicolumn{2}{c|}{100} & \multicolumn{2}{c}{150}\\
            ${\boldsymbol n}$  & 20 & 50 & 20 & 50 & 20 & 50\\
            \midrule
            {\bf Batch} & 85 & 227  & 246 & 617 & 428 & 2003 \\
            {\bf Online} & 36 & 86 & 108 & 273 & 227 & 776 \\
            \hline
        \end{tabular}
    \end{center}
\end{table}

Interestingly, the online approach is able to rapidly close the gap with batch LTL as the number of training tasks increases. 
This is particularly favorable since, from the computational perspective, online LTL is significantly faster than its batch counterpart. To further emphasize this aspect, Table~\ref{tab:computational-performance} compares the computational times required on average 
by online LTL and batch LTL as $T_{\rm tr}$ and $n$ vary. Online LTL is clearly faster then batch LTL. Indeed, as discussed in Sec.~\ref{sec:comparison-batch-computations}, whenever a new training task is provided, batch LTL requires to perform hundreds or thousands iterations of gradient descent to converge to a minimizer, even when ``warm restarting'' by initializing with the representation found at the previous step. On the other hand online LTL performs {\em a single gradient step} for each new task.

\begin{figure}[b!]
\centering
\includegraphics[width=0.58\columnwidth]{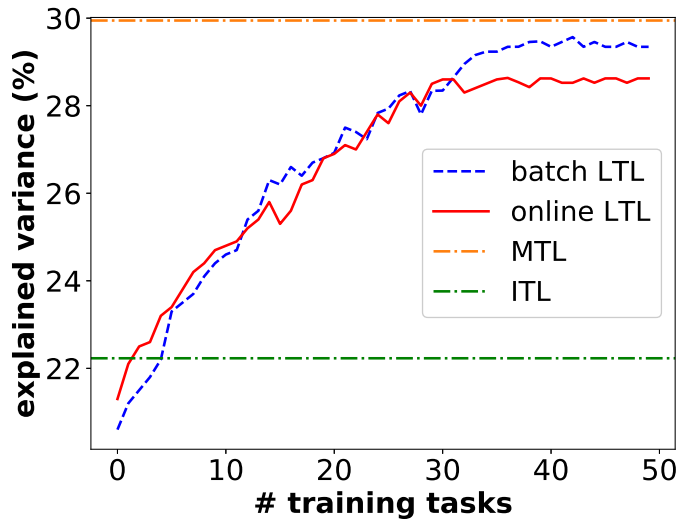}
\caption{Percentage explained variance of online LTL, batch LTL, ITL and MTL (on the test set) during one single trial of online learning on the Schools dataset as the number of training tasks increases incrementally.\label{fig:one-run-schools}
}
\end{figure}

\vspace{.2truecm}
\noindent {\bf Schools Dataset}. We evaluated online LTL on the Schools dataset, a dataset provided by the Inner London Education Authority (ILEA) and consisting of examination records from $139$ schools (see \cite{argyriou2008convex} for more details). Each school is associated to a regression task, individual students correspond to the input and their exam scores to the output. Input features belong to a $d=26$-dimensional input space $\X \subseteq \R^d$. 
We randomly sampled $25\%$ and $50\%$ of the $139$ tasks for LTL training and validation respectively and the remaining tasks were used as test set.
Figure~\ref{fig:one-run-schools} reports the performance of online LTL, batch LTL, ITL and MTL over one single run. 
Performance are reported in terms of the Explained Variance on the tasks \cite{argyriou2008convex} (higher values correspond to better performance). We note that the performance of the four compared methods are consistent with what we observed in the synthetic setting. In particular, online LTL is comparable to batch LTL. As expected, MTL outperforms all methods since it is able to exploit the relations between {\em test tasks}, which is not available to LTL algorithms.

\section{CONCLUSION AND FUTURE WORK}
\label{sec:conclusions}

In this work, we have proposed an incremental approach to LTL which
estimates a linear data representation
that works well on regression tasks coming from a meta distribution. 
{Compared with its batch (or offline) counterpart, this incremental approach is computationally more efficient both in terms of memory and number of operations, while enjoying the same generalization properties.}
Preliminary experiments 
have highlighted the favorable learning capability of the proposed learning-to-learn strategy.  To our knowledge this is the first efficient incremental algorithm for
meta learning for which statistical guarantees have been proved. Previous works either relied on algorithms which require to store the entire data sequence \cite{alquier2016regret} or which do not have statistical guarantees \cite{ruvolo2013ella}. Our analysis open several directions that will be worth investigating in the near future. First, it would be valuable to extend our analysis to a general class of loss functions. Although not allowing for a closed form expression as in the case of Ridge Regression, we suspect that it would be still be possible to extend our results by leveraging the regularity properties of the underlying learning algorithm. Second, we would like to depart from the feature learning setting by considering a more general family of learning-to-learn algorithms. Inspiration towards this direction is offered by the literature on multitask learning.


\appendix

\section*{APPENDIX}

\section{PROOF of \autoref{prop:least-squares}}

\label{Appendix A}

We denote by $\mathbb{S}^d$, $\mathbb{S}_+^d$ and $\mathbb{S}_{++}^d$ the sets of symmetric, positive semidefinite (PSD) and positive definite $d \times d$ real matrices, respectively. We denote by $\langle \cdot, \cdot \rangle$ the standard inner product in $\R^d$ (or $\R^n$, depending on the context) and by $\norma{\cdot}$ the associated norm. For any $p \in [1, \infty]$, the $p$-Schatten norm of a matrix
will be denoted  by $\norma{\cdot}_p$. Note that $\norma{\cdot}_1$, $\norma{\cdot}_2$ and $\norma{\cdot}_\infty$ are the trace, Frobenius and spectral norms, respectively.

Recall the definition of the function $\L_Z$  in Eq.~\eqref{eq:ridge-regression-obj-funct}. In order to provide the proof of \autoref{prop:least-squares}
we need the following Lemma. 
\begin{lemma}[Lemma 11 in \cite{maurer2005algorithmic}]\label{prop_inf_norm_tikh}
If $G_1, G_2 \in \mathbb{S}_{+}^d$, then for any $\gamma > 0$ and for $i = 1,2$, 
the following points hold.
\begin{enumerate}
\item[{\rm (a)}] $G_i + \gamma I$ is invertible.
\item[{\rm(b)}] $\big \| \bigl( G_i + \gamma I \bigr)^{-1} \big \|_\infty \le {\gamma}^{-1}$. 
\item[{\rm(c)}] $\big \| \bigl( G_1 + \gamma I \bigr)^{-1} - \bigl( G_2 + \gamma I \bigr)^{-1} \big \|_\infty 
\le{\gamma^{-2}} \big \| G_1 - G_2 \big \|_\infty$.
\item[{\rm(d)}] Let $w_1$ and $w_2$ satisfy $\bigl( G_i + \gamma I \bigr) w_i = \datay$ for some 
$\datay$, for $i = 1,2$. Then we have that
\begin{equation}
\Big |\norma{w_1}^2 - \norma{w_2}^2 \Big | \le 2 \gamma^{-3} 
\big \| G_1 - G_2 \big \|_\infty \norma{\datay}^2.
\end{equation}
\end{enumerate}
\end{lemma}

\begin{proof}[{\bf Proof of \autoref{prop:least-squares}}]
We now prove each point in turn.
\begin{enumerate}
\item Recall that a function $h: \mathbb{S}^d \rightarrow \mathbb{S}^d$ is matrix-convex if for every $A,B \in \mathbb{S}^d$ and $\lambda\in [0,1]$, $h(\lambda A + (1-\lambda)B) \preceq \lambda h(A) + (1-\lambda) h(B)$, see e.g. \cite[Chap.~V ]{bhatia2013matrix}. The function $h(A) = A^{-2}$ is matrix convex on $\mathbb{S}_{++}^d$. It follows, for every $\datay \in \R^n$, that the real-valued function $g_\datay: \mathbb{S}_{++}^d \rightarrow \R$ defined at $A \in \mathbb{S}_{++}^d$ as $g_\datay(A) =\langle \datay,A^{-2}\datay \rangle$ is convex. By Eq.~\eqref{eq:ridge-regression-obj-funct}, we have that $\L_Z(D)=g_\datay(X D X\trans + n I)$, hence it is convex because it is the composition of the convex function $g_\datay$ with an affine function.
\item Since the function $\L_Z$  in Eq.~\eqref{eq:ridge-regression-obj-funct}
is the composition of $\mathcal{C}^\infty$ functions, it is itself $\mathcal{C}^\infty$ on 
$\mathbb{S}_+^d$; therefore, as soon as we restrict it to a bounded subset of 
$\mathbb{S}_+^d$, all its derivatives\footnote{On the boundary of the set we define the derivatives by continuity.} 
 become Lipschitz. In this section we will use formula deriving from matrix calculus, we refer 
to the books \cite{kollo2006advanced, petersen2008matrix} for more details. Recalling the notation 
$M(D) = X D X\trans + n I \in \mathbb{R}^{n \times n}$, we now compute the Jacobian of the function 
$\L_Z$. Denoting by $x^k$ the $k$-th column 
of the matrix $X$ (it will be a column vector) for $k = 1, \dots, d$, we first show, 
for every $i, j \in \{1, \dots, d \}$,  that
\begin{equation}
\begin{split}
\bigl[ \nabla \L_Z(D) \bigr]_{i,j} & = - n ~ {\tr} \Big( \datay \datay \trans M(D)^{-1} 
\Big( x^i {x^j}\trans M(D)^{-1} + M(D)^{-1} x^i {x^j}\trans \Big)  M(D)^{-1} \Big) \\ 
& = - n \Big \langle \datay,  M(D)^{-1} \Bigl( x^i {x^j}\trans M(D)^{-1} + M(D)^{-1} x^i {x^j}\trans\Big) 
M(D)^{-1} \datay \Big \rangle.
\end{split}
\label{oooo}
\end{equation}
To see this, we first exploit the cyclic property of the trace to rewrite, for any $Z \in \mathcal{Z}^n$ and 
$D \in \mathbb{S}_+^d$, the function $\L_Z$ in Eq.~\eqref{eq:ridge-regression-obj-funct} as
\begin{equation}
\begin{split}
\L_Z(D) & = 
n~\big \langle \datay, M(D)^{-2} \datay \big \rangle
= n~{\tr}\Bigl( \datay \trans M(D)^{-2} \datay \Bigr) 
= n~{\tr}\Bigl( \datay \datay \trans M(D)^{-2}\Bigr)
= n~f \bigl( U(D) \bigr)
\nonumber
\end{split}
\end{equation}
where for any matrix $V \in \R^{n \times n}$ we have introduced the function
$f\bigl( V \bigr) = {\tr} (\datay \datay\trans V)$ and the symmetric matrix 
$U(D) = M(D)^{-2} \in \R^{n \times n}$. Hence, since $\displaystyle \frac{\partial f(V)}{\partial V} = 
\datay \datay \trans$ for any symmetric $V$ \cite[Eq.~(93)]{petersen2008matrix}, thanks to the chain rule
 \cite[Eq.~(126)]{petersen2008matrix}, for any $i, j \in \{1, \dots, d \}$, we have that
\begin{equation}
\frac{\partial \L_Z(D)}{ \partial D_{ij}} =
n~{\tr} \Bigl(\frac{\partial f(U(D))}{\partial U(D)}\trans \frac{\partial U(D)}{\partial D_{ij}} \Bigr) =
n~{\tr} \Bigl(\datay \datay \trans \frac{\partial U(D)}{\partial D_{ij}} \Bigr)
= n~\Big \langle \datay, \frac{\partial U(D)}{\partial D_{ij}} \datay \Big \rangle .
\end{equation}
Moreover,
the following formula, which is a direct consequence of Eq. ($33$) 
and Eq. ($53$) in \cite{petersen2008matrix}, holds:
\begin{equation}
\begin{split}
\frac{\partial M(D)^{-2}}{\partial D_{i,j}} 
& = - M(D)^{-1} \Bigl(\frac{\partial M(D)}{\partial D_{i,j}} M(D)^{-1} +
M(D)^{-1} \frac{\partial M(D)}{\partial D_{i,j}} \Bigr) M(D)^{-1}
\end{split}
\label{gggg}
\end{equation}
and since, for every $k, h \in \{1, \dots, n \}$, we have that
\begin{equation}
\Big[ \frac{\partial M(D)}{\partial D_{i,j}} \Big]_{kh} = 
\Big[\frac{\partial \bigl( X D X\trans \bigr)}{\partial D_{i,j}} \Big]_{kh}
= x^i_k x^j_h = \big[ x^i {x^j}\trans \big]_{kh}.
\end{equation}
Substituting in Eq.~\eqref{gggg} we obtain:
\begin{equation}
\begin{split}
\frac{\partial U(D)}{\partial D_{i,j}} & =
\frac{\partial M(D)^{-2}}{\partial D_{i,j}} 
= - M(D)^{-1} \Bigl( x^i {x^j}\trans
M(D)^{-1} + M(D)^{-1} x^i {x^j}\trans \Bigr) M(D)^{-1}
\end{split}
\end{equation}
and this conclude the proof of Eq.~\eqref{oooo}. 
Now, using the fact that for two $n \times 1$ vectors 
$v$ and $w$, we have that ${x^i}\trans v, {x^j}\trans w \in \mathbb{R}$ 
and
\begin{equation}
\bigl( {x^i}\trans v \bigr) \bigl( {x^j}\trans w \bigr) 
= \big[ X\trans v \big]_i \big[ X\trans w \big]_j 
= \big[ X\trans v w\trans X \big]_{ij},
\end{equation}
and exploiting the symmetry of $M(D)$, we can rewrite:
\begin{equation}
\begin{split}
\bigl[ \nabla \L_Z(D) \bigr]_{i,j}  & = 
- n ~\Big \langle \datay,  M(D)^{-1} \Bigl( x^i {x^j}\trans M(D)^{-1} + 
M(D)^{-1} x^i {x^j}\trans\Bigr)  M(D)^{-1} \datay \Big \rangle \\ 
& = - n ~\Big \langle \datay,  M(D)^{-1} x^i {x^j}\trans M(D)^{-2} \datay \Big 
\rangle - n \Big \langle \datay,  M(D)^{-2} x^i {x^j}\trans M(D)^{-1} \datay \Big \rangle \\
& = - n ~\Bigl( {x^i}\trans \underbrace{M(D)^{-1} \datay}_v \Bigr)
\Bigl( {x^j}\trans \underbrace{M(D)^{-2} \datay}_w \Bigr) - n ~
\Bigl( {x^i}\trans \underbrace{M(D)^{-2} \datay}_v \Bigr)  \Bigl( {x^j}\trans 
\underbrace{M(D)^{-1} \datay}_w \Bigr) \\
& = - n ~\Big[ X\trans M(D)^{-1} \datay \datay \trans M(D)^{-2} X \Big]_{ij} - n 
~\Big[ X\trans M(D)^{-2} \datay \datay \trans M(D)^{-1} X \Big]_{ij} \\
& = - n ~\Big[ X\trans M(D)^{-1} \Bigl( \datay \datay \trans M(D)^{-1} 
+ M(D)^{-1} \datay \datay \trans \Bigr) M(D)^{-1} X \Big]_{ij}.
\end{split}
\end{equation}
This last equation contains the elements of the Jacobian in the statement of the proposition.
\item In order to compute the Lipschitz constant of the function $\L_Z$ we first recall, 
for any $D\in \mathbb{S}_+^d$ and $Z \in \Z^n$, the expression 
$\L_Z(D) = n \big \| \bigl(X D X\trans + n I \bigr)^{-1} \datay \big\|^2$
in Eq.~\eqref{eq:ridge-regression-obj-funct}. Consequently,
for any $D_1, D_2 \in \mathbb{S}_+^d$ we have that
\begin{equation}
\begin{split}
\big | \L_Z(D_1) - \L_Z(D_2) \big | & = 
n  \Big | \big \| \bigl(X D_1 X\trans + n I \bigr)^{-1} \datay \big \|^2 - 
\big \| \bigl(X D_2 X\trans + n I \bigr)^{-1} \datay \big \|^2 \Big | \\
& \le \frac{2 n}{n^3} \big \| X D_1 X\trans -  X D_2 X\trans\big \|_\infty 
\norma{\datay}^2 \\
& = \frac{2 }{ n^2} \big \| X \bigl( D_1 -  D_2 \bigr) X\trans \big \|_\infty 
\norma{\datay}^2 \\
& \le \frac{2 }{n^2} \norma{X}_\infty^2  \norma{\datay}^2 \big \| D_1 -  D_2 \big \|_\infty \\
& \le \frac{2 }{n^2} \norma{X}_\infty^2  \norma{\datay}^2 \big \| D_1 -  D_2 \big \|_2,
\end{split}
\end{equation}
where in the first inequality we have applied \autoref{prop_inf_norm_tikh}-(d)
with $G_i = X D_i X\trans$, for $i = 1,2$. The statement now follows observing that
if $\mathcal{Y} \subseteq [0,1]$, then $\norma{\datay}^2 \le n$ and if $\mathcal{X} \subseteq
\mathcal{B}_1$, then $\norma{X}_\infty^2 \le n$.
\item We now compute the Lipschitz constant of the
gradient $\nabla \L_Z$. In the following
we will use the more compact notation $M_1 = M(D_1)$ and 
$M_2 = M(D_2)$, for any $D_1, D_2 \in \mathbb{S}_+^d$, 
and $R =  \datay \datay\trans$. Exploiting the following facts:
\begin{enumerate}
\item $\norma{AB}_2 \le \norma{A}_\infty \norma{B}_2$ for any
two matrices $A$ and $B$,
\item by \autoref{prop_inf_norm_tikh}-(b):
$\norma{M_i^{-1}}_\infty \le1/ n$ for $i = 1,2$,
\item by \autoref{prop_inf_norm_tikh}-(c):
\begin{equation*}
\begin{split}
\big \|  M_1^{-1} - M_2^{-1} \big \|_\infty & = 
\big \|  \bigl( X D_1 X\trans + nI \bigr)^{-1} - \bigl( X D_2 X\trans 
+ nI \bigr)^{-1} \big \|_\infty \\
& \le \frac{1}{n^2}\big \|  X D_1 X\trans- X D_2 X\trans\big \|_\infty \\
&\le \frac{1}{n^2} \norma{X}^2_\infty \big \|  D_1 -D_2 \big \|_\infty,
\end{split}
\end{equation*}
\item $\displaystyle \big \|  M_1^{-2} - M_2^{-2} \big \|_\infty = 
\big \| M_1^{-1} \bigl( M_1^{-1} - M_2^{-1} \bigr) +
\bigl( M_1^{-1} - M_2^{-1} \bigr) M_2^{-1} \big \|_\infty
\le \frac{2}{n} \big \|  M_1^{-1} - M_2^{-1} \big \|_\infty$,
\item if $\mathcal{X} \subseteq \mathcal{B}_1$ and $\mathcal{Y} \subseteq [0,1]$,
then $\norma{X}_2 \le \sqrt{n}$ and 
$\norma{R}_\infty = \norma{\datay \datay \trans}_\infty \le n$,
\end{enumerate}
we can write the following relations:
\begin{equation*}
\begin{split}
& \Big \| \nabla \L_Z(D_1) - \nabla \L_Z(D_2) \Big \|_2  = \\
& n  \Big \| X\trans \Bigl( M_1^{-1} \Bigl( \datay \datay \trans M_1^{-1} + M_1^{-1} \datay \datay \trans \Bigr) M_1^{-1}  
-  M_2^{-1} \Bigl( \datay \datay \trans M_2^{-1}  + M_2^{-1} \datay \datay \trans \Bigr) M_2^{-1} \Bigr) X  \Big \|_2 \le \\
& n \norma{X}_\infty \norma{X}_2 \Big \| M_1^{-1} \Bigl( \datay \datay \trans M_1^{-1} + M_1^{-1} 
\datay \datay \trans \Bigr) M_1^{-1} -  M_2^{-1} \Bigl( \datay \datay \trans M_2^{-1}  + M_2^{-1} \datay \datay \trans \Bigr) 
M_2^{-1} \Big \|_\infty \le \\
& n \norma{X}_\infty \norma{X}_2 \Big \| M_1^{-1} R M_1^{-2} + M_1^{-2} R M_1^{-1} -  
M_2^{-1} R M_2^{-2} -  M_2^{-2} R M_2^{-1} \Big \|_\infty = \\
& n \norma{X}_\infty \norma{X}_2 \Big \| M_1^{-1} R M_1^{-2} + M_1^{-2} R M_1^{-1} -  
M_2^{-1} R M_2^{-2} -  M_2^{-2} R M_2^{-1}  \\
& \quad \quad \quad \quad \quad \quad \quad \quad \quad \quad \quad \quad \quad \quad \quad 
\quad \quad \quad \quad \quad \quad \quad \pm M_2^{-1} R M_1^{-2} \pm  M_1^{-2} R M_2^{-1} \Big \|_\infty = 
\end{split}
\end{equation*}
\begin{equation*}
\begin{split}
& n \norma{X}_\infty \norma{X}_2 \Big \| \bigl( M_1^{-1} {-} M_2^{-1} \bigr) R M_1^{-2} 
+ M_1^{-2} R \bigl( M_1^{-1} {-} M_2^{-1} \bigr) + M_2^{-1} R \bigl( M_1^{-2} {-} M_2^{-2} \bigr)+ \\
& \quad \quad \quad \quad \quad \quad \quad \quad \quad \quad \quad \quad \quad \quad \quad 
\quad \quad \quad \quad \quad \quad \quad  \quad \quad \quad \quad \bigl( M_1^{-2} {-} M_2^{-2} \bigr) R M_2^{-1}\Big \|_\infty \le \\
& 2 \norma{X}_\infty \norma{X}_2 \norma{R}_\infty \Bigl( \frac{1}{n} \big \|  M_1^{-1} - M_2^{-1} 
\big \|_\infty + \big \|  M_1^{-2} - M_2^{-2} \big \|_\infty \Bigr) \le \\
& 2 \norma{X}_\infty \norma{X}_2 \Bigl(\big \|  M_1^{-1} - M_2^{-1} \big \|_\infty +
2 \big \|  M_1^{-1} - M_2^{-1} \big \|_\infty \Bigr) = \\
& 6 \norma{X}_\infty \norma{X}_2 \big \|  M_1^{-1} - M_2^{-1} \big \|_\infty \le  \frac{6}{n^2} \norma{X}_2 \norma{X}^3_\infty \big \|  D_1 -D_2 \big \|_\infty \\
& \le 6 \big \|  D_1 -D_2 \big \|_2.
\end{split}
\end{equation*}
\item The last point is contained in \cite[Prop. 1-(i)]{maurer2009transfer}; we report 
here the proof for completeness.  To this end, we require some additional notation, which will be also used also in the next section of the appendix.
\begin{remark}[Notation]
\label{notation}
According to our actual notation, $X \in \mathbb{R}^{n \times d}$ is the matrix
having as rows the points $x_i$ for $i = 1, \dots n$. In the sequel, since the
analysis will be extended also to the infinite dimension case, we will need to 
introduce the notation $\datax = (x_i)_{i=1}^n\in \mathcal{X}^n$, to indicate 
the collection of these points; to remark this difference, we will denote the complete
dataset $(\datax, \datay)$ by $\data$ and no more by $Z$. For any PSD 
matrix/ linear operator $D$ and any dataset $\data$, with some abuse of notation, 
we let $D^{1/2} \data = ( D^{1/2} x_i,y_i )_{i=1}^n$. Moreover, according 
to the notation introduced in the paper, we will denote the empirical error of 
a linear function $x \mapsto \langle w, x \rangle$ over the dataset $\data$ as
\begin{equation}
\hat \cR(\data, w) = \cR_\data(w).
\end{equation}
\end{remark}

Coming back to the proof of the proposition, as observed in 
\cite{argyriou2008convex, maurer2009transfer}, it is possible to rewrite the algorithm defined in 
Eq.~\eqref{eq:ridge-regression-closed-form} in the equivalent form
\begin{equation}
A_D(\data) = D^{1/2} A^{\rm Rid}(D^{1/2}\data),
\label{alg_feat_sel}
\end{equation}
where $A^{\rm Rid}(\data) \in \R^d$ is the solution of Ridge Regression (that is, using the square loss) on the dataset $\data$, that is
\begin{equation}
A^{\rm Rid}(\data) = \arg \min_{w\in \R^d} \Big \{ \hat \cR(\data, w) + \norma{w}^2 \Big \}.
\end{equation}
From Eq.~\eqref{alg_feat_sel}, we have that $\langle A_D(\data), x \rangle = 
\big \langle A^{\rm Rid}(D^{1/2} \data), D^{1/2} x \big \rangle$,
for any $x \in \mathcal{X}$ and any dataset $\data$. Consequently
\begin{equation}
\hat \cR \bigl( \data, A_D(\data) \bigr)  = \hat \cR \bigl( D^{1/2} \data, A^{\rm Rid}(D^{1/2} \data) \bigr).
\label{equi_loss_Tikh}
\end{equation}
Due to the definition of $A^{\rm Rid}$, assuming $\mathcal{Y} \subseteq [0,1]$, the following relations hold:
\begin{equation}
\begin{split}
\hat \cR \bigl( D^{1/2} \data, A^{\rm Rid}(D^{1/2} \data) \bigr) &  \le 
\hat \cR \bigl( D^{1/2} \data, A^{\rm Rid}(D^{1/2} \data) \bigr)
+ \big \| A^{\rm Rid}(D^{1/2} \data) \big \|^2 \\
& \le \frac{1}{n} \sum_{i = 1}^n \ell(0, y_i) = \frac{1}{n} \sum_{i = 1}^n y_i^2 \le 1.
\end{split}
\end{equation}
The claim now follows by combining the last inequality with Eq.~\eqref{equi_loss_Tikh}.
\end{enumerate}
\end{proof}


\section{UNIFORM BOUNDS for LINEAR FEATURE LEARNING}

\label{Appendix B}

In this section, we provide the uniform bounds on $\E(D) - \hat \E(D)$ and 
$\hat \E(D) - \hat \E_\mdata(D)$ (and the corresponding
symmetric quantities) for the family of linear feature learning algorithms. 
Our observations are essentially taken from \cite{maurer2009transfer}, 
we report them for clarity of exposition. We start from recalling some tools 
from empirical processes, then we state the uniform bounds for a more 
general class of learning algorithms and finally we specialize the bounds to 
linear feature learning. We ignore issues of measurability throughout.

\subsection{Preliminaries}

Let $m$ be a positive integer. In the following, we denote by $( \sigma_j )_{j =1}^m$ a sequence of  
i.i.d. Rademacher random variables, that is $\sigma_j$ takes values on $-1$ or $1$ with equal probabilities. 
We also denote by $(\gamma_j )_{j=1}^m$ a sequence of i.i.d. standard Gaussian random variables.
For a set $S \subseteq \mathbb{R}^m$ we define the Rademacher average of $S$ as
\[
\mathfrak{R}(S) = \mathbb{E}_{\sigma_j} \Big[ \sup_{v \in S} \frac{2}{m} \sum_{j = 1}^m \sigma_j v_j \Big]
\]
and the Gaussian average
\[
\mathcal{G}(S) = \mathbb{E}_{\gamma_j} \Big[ \sup_{v \in S} \frac{2}{m} \sum_{j = 1}^m \gamma_j v_j \Big].
\]
For more details about these quantities, we refer to \cite{bartlett2002rademacher}.
Given a class $\mathcal{F}$ of real-valued functions on a set $\mathcal{\mathcal{V}}$,
and given a point $V = (v_1, \dots, v_m) \in \mathcal{\mathcal{V}}^m$, we let 
\begin{equation}
\mathcal{F}(V) = \Big \{ \bigl(f(v_1) , \dots, f(v_m) \bigr) : f \in \mathcal{F} \Big \}
\subset \mathbb{R}^m
\end{equation}
so that $\mathfrak{R}(\mathcal{F}(V)) $ and $\mathcal{G}(\mathcal{F}(V))$
are the corresponding Rademacher and Gaussian averages.

The following theorem is taken from \cite{maurer2009transfer}, where the author considers
only the inequality for the function $\Phi_1$. Considering both inequalities allows 
us to obtain symmetric uniform bounds. 
The proof follows the same pattern as in \cite{maurer2009transfer}.

\begin{theoremshortref}[Theorem 4 in \cite{maurer2009transfer}]\label{main_theor}
Let $\eta$ be a probability 
distribution over the space $\mathcal{V}$, let 
$\mathcal{F}$ be a real-valued function class on $\mathcal{V}$
and let $V = (v_1, \dots, v_m) \in \mathcal{V}^m$. Define the random 
functions:
\begin{equation}
\begin{split}
\Phi_1(V) & = \sup_{f \in \mathcal{F}} \Big \{ \mathbb{E}_{v \sim \eta} 
\bigl[ f(v) \bigr] -  \frac{1}{m} \sum_{j = 1}^m f(v_j) \Big \} \\
\Phi_2(V) & = \sup_{f \in \mathcal{F}} \Big \{ \frac{1}{m} \sum_{j = 1}^m f(v_j) 
- \mathbb{E}_{v \sim \eta} \bigl[ f(v) \bigr] \Big \}.
\end{split}
\end{equation}
Then the following statements hold.
\begin{enumerate}
\item $\mathbb{E}_{V \sim \eta^m} \bigl[ \Phi_k(V) \bigr] \le 
\mathbb{E}_{V \sim \eta^m} \bigl[ \mathfrak{R}(\mathcal{F}(V)) \bigr]$, for  $k =1,2$.
\item If $\mathcal{F}$ is $[0,1]$-valued, then, for any $\delta\in(0,1]$, we have that
\begin{equation}
\Phi_k(V) \le \mathbb{E}_{V \sim \eta^m} \bigl[ \mathfrak{R}(\mathcal{F}(V)) \bigr]
+ \sqrt{\frac{\rm{log} \bigl( 1/\delta \bigr)}{2m}}
\end{equation}
with probability at least $1 - \delta$ in $V \sim \eta^m$, for $k =1,2$.
\item In the previous two points we can replace $\mathfrak{R}(\mathcal{F}(V))$
with $\displaystyle \sqrt{\pi/2} \mathcal{G}(\mathcal{F}(V))$.
\end{enumerate}
\begin{proof} 
The proof for the symmetric term $\Phi_2$ proceeds
in the same way as the one for $\Phi_1$ in \cite[Theorem 1]{maurer2009transfer},
more precisely, since the proof is based on symmetric arguments, the 
statement does not change if we flip the order of $\mathbb{E}_{v \sim \eta} 
\bigl[ f(v) \bigr]$ and $\frac{1}{m} \sum_{j = 1}^m f(v_j)$. The last inequality is a standard result, see e.g. \cite{boucheron2004concentration}.
\end{proof}
\end{theoremshortref}

\subsection{Uniform Bounds for a More General Family of Algorithms} 

The results presented in this sub-section hold for the infinite dimension case. 
In the sequel, we let $\X$ be a generic Hilbert space and we denote by
$\langle \cdot, \cdot \rangle$ and $\norma{\cdot}$ its scalar product 
and the induced norm. We let $\mathcal{S}_+(\mathcal{X})$ be the 
set of positive semidefinite bounded linear operators on $\X$ and,
for any operator $D \in \mathcal{S}_+(\mathcal{X})$, we denote
its $p$-Schatten norm by $\norma{D}_p$, where $p \in [1, \infty]$.
We continue to use the notation introduced in the paper and in Remark \ref{notation},
in particular, $\mdata = \{ \data_t \}_{t=1}^T$ is the meta sample and,
for any $D \in \mathcal{S}_+(\mathcal{X})$, we denote $D^{1/2} \data 
= \bigl( D^{1/2} x_i,y_i \bigr)_{i=1}^n$.
Throughout this section we will consider linear models and a learning algorithm 
$A(\data)$ processing a training set $\data \in \Z^n$ of $n$ points:
\begin{equation}
\begin{split}
A: ~& \Z^n \to \mathcal{X} \\
& ~\data \mapsto A(\data),
\end{split}
\end{equation}
hence, according to our notation, we have that 
$A(\data)(x) = \langle A(\data), x \rangle$ for any $x \in \mathcal{X}$. 
For any $D \in \mathcal{S}_+(\mathcal{X})$, define now the more 
general family of modified algorithms
\begin{equation}
A_D(\data) = D^{1/2} A(D^{1/2} \data).
\end{equation}
By this definition, as we have already 
observed in the proof of \autoref{prop:least-squares}-(5) in Sec. \ref{Appendix A}, 
we have that 
\begin{equation}
\big \langle A_D(\data), x \big \rangle = \big \langle A(D^{1/2}
\data), D^{1/2} x \big \rangle
\end{equation}
for any $x \in \mathcal{X}$ and consequently
\begin{equation}
\hat \cR \bigl( \data, A_D(\data) \bigr) = \hat \cR \bigl( D^{1/2} \data, A(D^{1/2} \data) \bigr).
\end{equation}
In this way, we can consider the family of learning algorithms $\big \{ \data \mapsto A_D (\data): 
D \in \mathcal{S}_+(\mathcal{X}) \big \}$, parameterized by 
the operators $D$. Recall now, for every $D \in \mathcal{S}_+(\mathcal{X})$, the notion of transfer risk
\[
{\mathcal E}(D) = \mathbb{E}_{\mu \sim \rho} \mathbb{E}_{\data \sim \mu^n}  
\mathbb{E}_{(x,y) \sim \mu}  \bigl[ \ell (\langle A_D(\data), x \rangle, y) \bigr],
\]
future empirical risk
\[
\hat{\mathcal E}(D) = \mathbb{E}_{\mu \sim \rho} \mathbb{E}_{\data \sim \mu^n}  
\bigl[ {\hat \cR} \bigl(\data, A_D(\data) \bigr) \bigr]
\]
and multi task empirical risk
\[
\hat{\mathcal E}_\mdata(D) =  \frac{1}{T} \sum_{t = 1}^T {\hat \cR}\bigl(\data_t, A_D(\data_t) \bigr).
\]

The following two theorems are taken from \cite{maurer2009transfer},
where the author does not consider the symmetric case, which
immediately follows from \autoref{main_theor}. In the sequel, the
symbol $C_{\rho}$, already introduced in the paper,  
denotes the covariance of the input data, obtained by  
averaging over all input marginals sampled from $\rho$, that is, 
$C_{\rho} = \Exp_{\mu\sim\rho}\Exp_{(x,y)\sim\mu} ~ 
[C(x)]$, where for any $x \in \mathcal{X}$, 
and for any $v \in \mathcal{X}$, $C(x) v = \langle v, x \rangle x$.

\begin{theoremshortref}[Theorem 6 in \cite{maurer2009transfer}]\label{uniform_bound_general_1}
Let $p$ and $q$ be conjugate exponents in $[1, \infty]$ and assume $\mathcal{X} \subseteq \mathcal{B}_1$. 
Consider a learning algorithm $A$ such that $\norma{A(D^{1/2}\data)} \le 1$ for any 
$\data \in \Z^n$ and any $D \in \mathcal{S}_+(\X)$, and let $\ell$ be a loss function 
such that, for any $y \in \mathbb{R}$, $\ell(\cdot, y)$ has Lipschitz constant $L(K)$ on the interval
$[-K,K]$, for any $K \ge 0$. Then for any meta distribution $\rho$ on $\Z$ and for any 
$D \in \mathcal{S}_+(\X)$ we have that:
\begin{equation}
\big | \mathcal{E}(D) - \hat{\mathcal{E}}(D) \big | \le
\sqrt{\frac{2 \pi}{n}} L \bigl(\norma{D}_\infty^{1/2} \bigr) 
{\norma{C_{\rho}}_p}^{1/2} {\norma{D}_q}^{1/2}.
\end{equation}
\end{theoremshortref}

In order to give the next theorem, we need to introduce the 
Gramian matrix defined by the entries $[G(\datax)]_{i,j} = 
\langle x_i, x_j \rangle$ for $i, j = 1, \dots, n$.

\begin{theoremshortref}[Theorem 8 in \cite{maurer2009transfer}]\label{uniform_bound_general_2}
Let $\mathcal{X} \subseteq \mathcal{B}_1$ and $\mathfrak{D} \subseteq \mathcal{S}_+(\X)$ be a bounded set.
Consider a function $f: \Z^n \to [0,1]$ satisfying the condition
\begin{equation}
\big | f(\data) - f(\data') \big | \le \frac{L_K}{n} \big \| G(\datax) - G(\datax') \big \|_2
\label{kernel_stab}
\end{equation}
for any $\data, \data' \in \Z^n$ and for some $L_K \ge 0$. Let $\mu_1,\dots,\mu_T$ tasks independently 
sampled from $\rho$ and $\data_t$ sampled from $\mu_t^n$ for $t\in\{1,\dots,T\}$. Then, for any $\delta\in(0,1]$,
we have that
\begin{equation}
\sup_{D \in \mathfrak{D}} \Big | \mathbb{E}_{\mu \sim \rho} \mathbb{E}_{\data \sim \mu^n} \bigl[ f(D^{1/2} \data) \bigr] -
\frac{1}{T} \sum_{t = 1}^T f(D^{1/2} \data_t) \Big | \le 
\Bigl( \sup_{D \in \mathfrak{D}} \norma{D}_2 \Bigr) \frac{\sqrt{2 \pi} L_K}{\sqrt{T}} 
+ \sqrt{\frac{\rm{log} \bigl( 1/\delta \bigr)}{2T}}
\end{equation}
with probability at least $1-\delta$. 
\end{theoremshortref}

\subsection{Application to the Family of Linear Feature Learning Algorithm} 
\label{app:B3}
Similarly to what observed in \autoref{prop:least-squares}-(5) in Sec. \ref{Appendix A}, 
also in the infinite dimension case, we can cast the family of linear feature learning 
algorithms in the framework described in the previous sub-section, taking the original 
vanilla algorithm $A(\data)$ as Ridge Regression with regularization parameter equal to $1$:
\begin{equation}
A(\data) = A^{\rm Rid}(\data) = \arg \min_{w} \Big \{ \hat{\cR}(\data, w) 
+ \norma{w}^2 \Big \},
\label{Tikh_app}
\end{equation}
we refer to \cite{maurer2009transfer} for more details.
Thus, we can apply the results in the previous sub-section to this specific case, 
in order to obtain the results stated in the paper for the uniform bounds. In fact,
in the paper we have analyzed the finite dimension case, but from this analysis,
we deduced that they still hold in the infinite dimension setting. The following 
definition will be used in the sequel.

\begin{definition}[Definition 1 in \cite{maurer2009transfer}] Relative to a loss function 
$\ell$, a learning algorithm $A: \Z^n \to \mathcal{X}$ is said to 
\begin{enumerate}
\item be 1-bounded if $\norma{A(\data)} \le 1$ and $\hat{\cR}(\data, A(\data)) \le 1$
for any $\data \in \Z^n$;
\item have kernel stability $L_K$ if 
$\displaystyle \big | \hat{\cR}(\data, A(\data)) - \hat{\cR}(\data', A(\data')) \big | \le 
\frac{L_K}{n} \big \| G(\datax) - G(\datax') \big \|_2$,
for any $\data, \data' \in \Z^n$ and for some $L_K \ge 0$.
\end{enumerate}
\label{stable_alg}
\end{definition}

The following two lemmas are essentially taken from \cite{maurer2009transfer}
and they are respectively immediate consequences of 
\autoref{uniform_bound_general_1} and \autoref{uniform_bound_general_2} 
applied to the family of linear feature learning algorithms with restriction to the set
$$
\mathfrak{D} = \Dla = \big \{ D \in \mathcal{S}_+
(\mathcal{X}): {\tr}(D) \le 1/\lambda \big \}.
$$

\uniformboundfirst*
\begin{proof} 
Thanks to the assumption $\Y \subseteq [0,1]$, by \cite[Prop.~1]{maurer2009transfer}, $A^{\rm Rid}(\data)$, is $1$-bounded -- and in particular,
$\norma{A^{\rm Rid}(D^{1/2} \data)} \le 1$ for any $D \in \mathcal{L}_+(\X)$ and any dataset $\data$ -- with respect to the 
square loss. Hence, we can apply \autoref{uniform_bound_general_1} 
to $A^{\rm Rid}$. We restrict to
the set $\Dla$, we choose $q = 1$ and $p = \infty$
and we observe that the square loss is $M(K) = 
2(K + 1)$-Lipschitz on the interval $[-K, K]$.
\end{proof}

\begin{proposition}\label{prop:uniform-generalization-bound-2}
Let $\mathcal{X} \subseteq \mathcal{B}_1$, $\Y\subseteq[0,1]$ and let $\ell$ be the square loss. Let $\mu_1,\dots,\mu_T$ be independently sampled from $\rho$ and $Z_t$ sampled from $\mu_t^n$ for $t\in\{1,\dots,T\}$. Then, for any $\delta\in(0,1]$,
we have, with probability at least $1-\delta$, that
$$
\sup_{D\in\Dla} \big | \hat{\mathcal{E}}(D) - \hat{\mathcal E}_\mdata(D) \big |
\le \frac{2\sqrt{2 \pi}}{\lambda \sqrt{T}}
+ \sqrt{ \frac{\rm{log}\bigl( 1/\delta \bigr)}{2T}}.
$$
\begin{proof} 
Thanks to the assumption that $\Y \subseteq [0,1]$, by \cite[Prop.~1]{maurer2009transfer}, $A^{\rm Rid}(\data)$ is $1$-bounded -- and in particular,
$\hat{\cR} \bigl( D^{1/2}\data, A^{\rm Rid}(D^{1/2} \data) \bigr) \le 1$ for any $D \in \mathcal{L}_+(\X)$ and any dataset $\data$ -- and has kernel stability 
$L_K = 2$ with respect to the square loss.  We can then apply \autoref{uniform_bound_general_2} to the function 
$$
f(\data) = \hat{\cR} \bigl( \data, A_D(\data)\bigr)
= \hat{\cR} \bigl( D^{1/2}\data, A^{\rm Rid}(D^{1/2} \data) \bigr).
$$
\end{proof}
\end{proposition}


\section{PROOF of \autoref{thm:main-batch}}

\label{Appendix C}

In this section, we report the proof of \autoref{thm:main-batch}. We do not 
make any claim of originality in this theorem which is merely a collection of results
contained in \cite{maurer2009transfer}; we report the proof 
for completeness.

\batchtheorerecall*

\begin{proof} Similarly to the online case, the proof of \autoref{thm:main-batch}
relies on the following decomposition.
$$
\E(\hat{D}_T) - \mathcal{E}(D_*)  =
\underbrace{\E(\hat{D}_T) - \hat{\E}_\mdata(\hat{D}_T)}_ {A}
+ \underbrace{\hat{\E}_\mdata(\hat{D}_T) - \hat{\E}_\mdata
(D_*)}_{B} + \underbrace{\hat{\E}_\mdata(D_*) - \E(D_*)}_{C}.
$$
We now describe how to deal with each term. We decompose the term $A$ as 
\begin{equation}
\E(\hat{D}_T) - \hat{\E}_\mdata(\hat{D}_T) =
\underbrace{\E(\hat{D}_T) - \hat{\E}(\hat{D}_T)}
_{A1} + \underbrace{\hat{\E}(\hat{D}_T) - \hat{\E}_\mdata
(\hat{D}_T)}_ {A2}
\end{equation}
and we bound the term $A1$ by \autoref{prop:uniform-generalization-bound}
and the term  $A2$ by \autoref{prop:uniform-generalization-bound-2} with confidence 
parameter $\delta/2$. The term $B$, thanks to the definition of $\hat{D}_T$,
is negative. Lastly, as regards the term $C$, we split it in
\begin{equation}
\hat{\E}_\mdata(D_*) - \E(D_*) =
\underbrace{\hat{\E}_\mdata(D_*) - \hat{\E}(D_*)}_{C1}  
+ \underbrace{{\hat{\E}}(D_*) - {\E}(D_*)}_{C2},
\end{equation}
where we bound $C2$ by \autoref{prop:uniform-generalization-bound},
while, in order to bound the first term $C1$, we apply Hoeffding's 
inequality (see \autoref{Hoeff} below) with parameters $a_t = 0$ and $b_t = 1$
for any $t$ (thanks to \autoref{prop:least-squares}-(5)) and confidence parameter $\delta/2$, 
i.e. for any $\delta \in (0,1]$, we have that
\begin{equation}
\hat{\E}_\mdata(D_*) - \hat{\E}(D_*) \le
\sqrt{\frac{\rm{log} \bigl( 2/\delta \bigr)}{2T}}
\end{equation}
with probability at least $1 - \delta/2$ in $\mdata$.
Joining all the previous parts, the statement follows.
\end{proof}

\begin{lemma}[Hoeffding's inequality \cite{boucheron2004concentration}]\label{Hoeff}
Let $m$ be a positive integer and let $X_1, \dots, X_m$ be independent random variables 
such that $X_i \in [a_i, b_i]$ with probability $1$, for $i = 1, \dots, m$. Define 
$\bar{X}_m = \displaystyle \frac{1}{m} \sum_{i = 1}^m X_i$. Then, for any 
$\epsilon > 0$, we have that
\begin{equation}
\mathbb{P} \bigl[ \bar{X}_m - \mathbb{E} \bigl[ \bar{X}_m \bigr] \ge \epsilon \bigr]
\le {\rm{exp}} \Bigl( - \frac{2 m^2 \epsilon^2}{\sum_{i = 1}^m (b_i -a_i)^2} \Bigr),
\end{equation}
or equivalently, for any $\delta \in (0,1]$, we have that
\begin{equation}
\bar{X}_m - \mathbb{E} \bigl[\bar{X}_m\bigr] \le \sqrt{\frac{1}{2 m^2}
\Bigl( \sum_{i = 1}^m (b_i -a_i)^2 \Bigr) {\rm{log}} \Bigl( \frac{1}{\delta} \Bigr)}
\end{equation}
with probability at least $1 - \delta$.
Moreover, thanks to symmetric arguments, the previous inequalities
hold also for $\mathbb{E} \bigl[ \bar{X}_m \bigr] - \bar{X}_m$.
\end{lemma}


\section{NON-ASYMPTOTIC RATES for PROJECTED STOCHASTIC SUBGRADIENT ALGORITHM}

In this section, we briefly describe how to derive non-asymptotic
convergence rates in probability for Projected Stochastic Subgradient Algorithm 
(PSSA), exploiting the regret bounds for Projected Online Subgradient 
Algorithm (POSA). In the first part we give a regret
bound for POSA and we specialize it to Algorithm~\ref{alg:general-pssa} 
for the case of the square loss (\autoref{lem:regret-bound}). In the second 
part we first show, in general, how 
a bound on the regret implies a rate in probability for the convergence
in the statistical setting and then we specialize this result to obtain the 
bound on the excess empirical future risk of the output of 
Algorithm~\ref{alg:general-pssa} for the case of the square loss 
(\autoref{prop:bound-empirical-transfer-error}). 
The results contained in this section are standard, 
we will cite during the presentation some references where the interested
reader can find more details. Throughout this section, no differentiability 
assumptions on the functions will be made, 
we only require them to be convex and Lipschitz. We also require the 
boundedness of the diameter of the set over which we optimize. The 
general analysis will be conducted in a Hilbert space with scalar 
product $\langle \cdot, \cdot \rangle$ and induced norm $\norma{\cdot}$. 

\subsection{Projected Online Subgradient Algorithm, POSA}
\label{app:D1}

The Online Convex Optimization (OCO) framework \cite{hazan2016introduction} 
over a convex and closed set $H$ of a Hilbert space can be seen as a 
repeated game: at iteration $t$, the online player, i.e. the online
algorithm, chooses $h^{(t)} \in H$, after this, a cost function 
$f_t: H \to \mathbb{R}$ is revealed by the adversary and the cost 
incurred by the online player is $f_t(h^{(t)})$. 
The cost functions $f_t$ are usually assumed to be bounded 
convex functions over $H$, belonging to some bounded 
family of functions and they could be even adversely chosen. 
The performance of an online algorithm  
over a total number of game iterations $T$ 
is measured by its $regret$, defined as the difference between the 
total averaged cost the algorithm incurred over $T$ matches and 
that of the best fixed decision in hindsight:
\begin{equation}
R_T = \frac{1}{T} \sum_{t = 1}^T f_t(h^{(t)}) - \min_{h \in H} 
\frac{1}{T} \sum_{t = 1}^T f_t(h).
\end{equation}
In the sequel, we will always assume the convexity of the functions
$f_t$ and the existence of a minimizer of the batch 
problem $\hat{h} \in \arg \min_{h \in H} \sum_{t = 1}^T f_t(h)$.
In our case, we will focus on the classical Projected Online Subgradient Algorithm  
described in Algorithm~\ref{OGDA} and we will give an upper bound on its regret.
When needed, the following assumptions will be made. 

\begin{assumption} 
Assume that for any $t$ the functions $f_t$ are $G$-Lipschitz on $H$,
i.e. there exists a positive constant such that $\norma{u} \le G$ for any
$u \in \partial f_t(h)$ and for any $h \in H$.
\label{ass_Lipschitz}
\end{assumption}

\begin{assumption}
Assume that the diameter of the set $H$ is 
bounded by some constant $\mathcal{D}> 0$, i.e.
$\displaystyle \sup_{h, h' \in H} \norma{h -h'} \le \mathcal{D}$.
\label{ass_diameter}
\end{assumption}

The following theorem is a classical result and a
slightly different version can be found in \cite[Thm.~3.1]{hazan2016introduction}, we report 
here the proof because of clarity and completeness.

\begin{algorithm}[t]
\caption{POSA}\label{OGDA}
\begin{algorithmic}
\State ~
   \State {\bfseries Input:} $T\in \mathbb{N}$ number of iterations, $\{\gamma_t \}_t$ step sizes 
   \State {\bfseries Initialization:} $h^{(1)} \in H$
   \State {\bfseries For} $t=1$ to $T$
   \State \qquad Receive ~~ $f_t $, pay $f_t(h^{(t)})$
   \State \qquad Choose ~~ $u_t \in \partial f_t(h^{(t)})$
   \State \qquad Update ~~ $h^{(t+1)} = {\rm proj}_{H} 
(h^{(t)} - \gamma_t u_t)$
 \State {\bfseries Return $h^{(T)}$} 
\State ~
\end{algorithmic}
\end{algorithm}

\begin{theoremshortref}[Regret Bound for Algorithm~\ref{OGDA}] \label{regret_OGDA}
Under \autoref{ass_Lipschitz} and \autoref{ass_diameter}, the regret of 
Algorithm~\ref{OGDA}, with $\gamma_t = c/\sqrt{t}$ for some $c > 0$, 
is  bounded by
\begin{equation}
R_T \le \frac{1}{2} \Bigl( \frac{\mathcal{D}^2}{c} + 2 c G^2 \Bigr) \frac{1}{\sqrt{T}}.
\end{equation}
Moreover, the optimal value for the previous bound, attained at $c = \displaystyle \frac{\mathcal{D}}{\sqrt{2} G}$, is
$\displaystyle R_T \le \frac{\sqrt{2} \mathcal{D} G}{\sqrt{T}}$. 
\begin{proof}
Since $u_t \in \partial f_t(h^{(t)})$, by convexity of $f_t$
and definition of subgradient, we have that:
\begin{equation}
f_t(h^{(t)}) - f_t(\hat{h}) \le \langle u_t, h^{(t)} - \hat{h} \rangle.
\label{1}
\end{equation}
Using the update rule of Algorithm~\ref{OGDA}, Pythagorean Theorem (i.e. the
non-expansiveness property of the projection operator) and 
\autoref{ass_Lipschitz}, the following relations hold:
\begin{equation}
\begin{split}
\norma{h^{(t+1)} - \hat{h}}^2 & = \norma{{\rm{proj}}_H (h^{(t)} - \gamma_t u_t) - \hat{h}}^2 \\
& \le \norma{h^{(t)} - \gamma_t u_t - \hat{h} }^2 \\
& = \norma{h^{(t)} - \hat{h}}^2 - 2 \gamma_t \langle u_t, h^{(t)} - \hat{h} \rangle + 
\gamma_t ^2 \norma{u_t}^2 \\
& \le \norma{h^{(t)} - \hat{h}}^2 - 2 \gamma_t \langle u_t, h^{(t)} - \hat{h} \rangle + 
\gamma_t ^2 G^2,
\end{split}
\end{equation}
which imply that
\begin{equation}
\langle u_t, h^{(t)} - \hat{h} \rangle  \le \frac{\norma{h^{(t)} - \hat{h}}^2 - 
\norma{h^{(t + 1)} - \hat{h}}^2}{2 \gamma_t} + \frac{\gamma_t G^2}{2}.
\label{2}
\end{equation}
Combining Eq.~\eqref{2} with Eq.~\eqref{1}, we obtain:
\begin{equation}
f_t(h^{(t)}) - f_t(\hat{h}) \le \frac{\norma{h^{(t)} - \hat{h}}^2}{2 \gamma_t} 
- \frac{\norma{h^{(t + 1)} - \hat{h}}^2}{2 \gamma_t}+ \frac{\gamma_t G^2}{2}.
\label{3}
\end{equation}
Now, summing Eq.~\eqref{3} from $t = 1$ to $t = T$, using the convention $1/\gamma_0 = 0$,
and setting $\gamma_t = c/\sqrt{t}$ we can write:
\begin{equation}
\begin{split}
\sum_{t = 1}^T \Bigl( f_t(h^{(t)}) - f_t(\hat{h}) \Bigr) & \le 
\frac{1}{2} \sum_{t = 1}^T \frac{\norma{h^{(t)} - \hat{h}}^2}{\gamma_t} 
- \frac{1}{2} \sum_{t = 1}^T \frac{\norma{h^{(t + 1)} - \hat{h}}^2}{\gamma_t} +
\frac{G^2}{2} \sum_{t = 1}^T \gamma_t \\
& = \frac{1}{2} \sum_{t = 1}^T \frac{\norma{h^{(t)} - \hat{h}}^2}{\gamma_t} 
- \frac{1}{2} \sum_{t = 1}^T \frac{\norma{h^{(t)} - \hat{h}}^2}{\gamma_{t-1}} 
- \frac{1}{2} \frac{\norma{h^{(T + 1)} - \hat{h}}^2}{\gamma_T} +
\frac{G^2}{2} \sum_{t = 1}^T \gamma_t \\
& \le \frac{1}{2} \sum_{t = 1}^T \Bigl(\frac{1}{\gamma_t} - \frac{1}{\gamma_{t-1}} \Bigr) 
\norma{h^{(t)} - \hat{h}}^2 + \frac{G^2}{2} \sum_{t = 1}^T \gamma_t \\
& \le \frac{1}{2} \Bigl( \frac{\mathcal{D}^2}{\gamma_T} + {G^2}\sum_{t = 1}^T \gamma_t \Bigr)
\le \frac{1}{2} \Bigl( \frac{\mathcal{D}^2}{c} + 2 c G^2 \Bigr) \sqrt{T},
\end{split}
\end{equation}
where we have exploited \autoref{ass_diameter}, more precisely $\norma{h^{(t)} - \hat{h}} \le \mathcal{D}$, the fact that $\sum_{t = 1}^T \Bigl(\frac{1}{\gamma_t} - \frac{1}{\gamma_{t-1}} \Bigr) = \frac{1}{\gamma_T}$ and the inequality $\sum_{t = 1}^T \frac{1}{\sqrt{t}} \le 2 \sqrt{T} - 1 \le 2 \sqrt{T}$. Dividing by $T$ and optimizing with respect to $c$, the result follows.
\end{proof}
\end{theoremshortref}


We now specialize the regret bound obtained for the generic Algorithm~\ref{OGDA} to
our Algorithm~\ref{alg:general-pssa} described in the paper for the square loss.

\regretboundour*
\begin{proof} 
The thesis follows from applying \autoref{regret_OGDA} 
to the context of Algorithm~\ref{alg:general-pssa} {with the square loss}. In this case 
the iteration $h^{(t)}$ coincide with $D^{(t)}$, the cost 
functions are identified with $f_t = \L_{Z_t}$, hence they 
are $2$-Lipschitz thanks to \autoref{prop:least-squares}-(3)
and, consequently, we can take $G = 2$ in \autoref{regret_OGDA}. 
Moreover, the diameter $\mathcal{D}$ of the set over which 
we project $\Dla$ (in the previous notation $H$) is $2/\lambda$. Indeed, for any $D \in \Dla$ we have that 
$\norma{D}_2 \le \norma{D}_1 = {\tr} (D) \le 1/\lambda$, hence 
$\mathcal{D} = \sup_{D, D' \in \Dla} \norma{D - D'}_2 \le 2/\lambda$. 
\end{proof}


\subsection{Online-to-Batch Conversion}
\label{app:D2}
Consider a collection of data points $\{Z_t \}_t$ belonging to
some space and let $\eta$ be a probability distribution over it. 
In the sequel of the discussion we will ignore all measurability issues.
Let $H$ be a set as above and for every 
$h \in H$ define $F(h) = \mathbb{E}_{Z \sim \eta} \bigl[ \L_Z(h) \bigr]$,
where, for any $Z$, $\L_Z$ is a convex function. In the following we will 
consider the optimization problem
\begin{equation}
\min_{h \in H} F(h)
\label{probl_stoch}
\end{equation}
and we will assume the existence of a minimizer $h_* \in \arg \min_{h \in H} F(h)$. 
In order to solve the stochastic problem in Eq.~\eqref{probl_stoch}, we will analyze 
the general incremental procedure described in Algorithm~\ref{OA}, where the next point is 
updated by some rule depending on the past history of the process, for instance, if we 
choose the update $h^{(t+1)} = {\rm{proj}}_H (h^{(t)} - \gamma_t u_t)$, for some 
$\gamma_t > 0$ and $u_t \in \partial f_t(h^{(t)})$, then Algorithm~\ref{OA} coincides 
with POSA (Algorithm~\ref{OGDA}) applied to the functions $f_t = \L_{Z_t}$. 

\begin{algorithm}[t]
\caption{Generic Incremental Procedure in the Online and Statistical Settings} \label{OA}
\begin{multicols}{2}
\begin{algorithmic}
\State ONLINE SETTING
\State ~
   \State {\bfseries Input:} $T\in \mathbb{N}$ number of iterations, $\{\gamma_t \}_t$ step sizes 
   \State {\bfseries Initialization:} $h^{(1)} \in H$
   \State {\bfseries For} $t=1$ to $T$:
   \State \qquad Receive ~~ $Z_t$ $\longrightarrow$ \textbf{no further assumptions}
   \State \qquad Define ~~ $f_t = \L_{Z_t}$, pay $f_t(h^{(t)})$
   \State \qquad Update ~~ $h^{(t+1)}$
 \State {\bfseries Return $h^{(T)}$} 
\State ~
\State STATISTICAL SETTING
\State ~
   \State {\bfseries Input:} $T\in \mathbb{N}$ number of iterations, $\{\gamma_t \}_t$ step sizes 
   \State {\bfseries Initialization:} $h^{(1)} \in H$
   \State {\bfseries For} $t=1$ to $T$:
   \State \qquad Receive ~~ $Z_t$ $\longrightarrow$ \textbf{sampled i.i.d. from \bm{$\eta$}}
   \State \qquad Define ~~ $f_t = \L_{Z_t}$, pay $f_t(h^{(t)})$
   \State \qquad Update ~~ $h^{(t+1)}$
 \State {\bfseries Return $\bar{h}_T = \frac{1}{T} \sum_{t = 1}^T h^{(t)}$} 
\State ~
\end{algorithmic}
\end{multicols}
\end{algorithm}

In the online setting no further assumptions about the data are made, however, in 
the statistical setting we typically assume that the data are i.i.d. from the distribution 
$\eta$; since this last setting is more restrictive, one would expect that  
if Algorithm~\ref{OA} solves the problem in the online framework, i.e. if its regret $R_T$ 
is such that $R_T \to 0$ as $T \to \infty$, then it will also solve the corresponding problem 
\eqref{probl_stoch} in the statistical setting. This statement is formally confirmed by the following
theorem, which relies on results taken from \cite{cesa2004generalization}.

\begin{theoremshortref}[Theorem $9.3$ in \cite{hazan2016introduction}]\label{sgd_rate_Hazan}
{Let $f_t = \L_{Z_t}$ be convex functions with values in $[0,1]$ for any $Z_t$,
$t \in \{1, \dots, T\}$} and let the points $\{ Z_t \}_{t = 1}^T$ processed by 
Algorithm~\ref{OA} be i.i.d. from $\eta$. Then, denoting by $R_T$
the regret bound of Algorithm~\ref{OA}, for any $\delta \in (0,1]$,
we have that
\begin{equation}
F(\bar{h}_T) - F(h_*) \le R_T + \sqrt{\frac{8 \rm{log}(2/\delta)}{T}}
\end{equation}
with probability at least $1 - \delta$ in the sampling of the points $\{Z_t \}_{t =1}^T$.
\end{theoremshortref}

The previous theorem relies on the theory of Martingales \cite{grimmett2001probability} 
and the analysis of the first term $\frac{1}{T} \sum_{t = 1}^T 
f_t(h^{(t)})$ of the regret (\cite{cesa2004generalization}), in fact this term
is a data-dependent statistics evaluating the average cumulative error 
of the prediction $h^{(t)}$ of the algorithm on the next point $Z_t$, 
therefore it is reasonable to expect that it contains information about 
the generalization ability of the algorithm.

Adapting the previous discussion to the setting of Algorithm~\ref{alg:general-pssa} for the
square loss, we obtain the following rate for the excess empirical future risk of 
online estimator returned by the algorithm.

\sgdrateour*

\begin{proof} The statement directly follows by combining \autoref{sgd_rate_Hazan} with 
the regret bound in \autoref{lem:regret-bound} to the context of Algorithm~\ref{alg:general-pssa}
for the square loss: we identify the set $H$ with the set $\Dla$, the output $\bar{h}_T$ with
the online estimator $\bar{D}_T$, the expectation $\mathbb{E}_{\eta}$ with 
$\mathbb{E}_{\mu \sim \rho} \mathbb{E}_{Z \sim \mu^n}$
and the function $F$ with the future empirical risk $\hat \E$, the remaining identifications
are obvious. We remark that, thanks to \autoref{prop:least-squares}-(5), the boundedness 
condition on the functions $\L_{Z_t}$ needed in order to apply \autoref{sgd_rate_Hazan}, is satisfied in our setting.
\end{proof}


\section{PROJECTION ON THE SET \texorpdfstring{$\Dla$}{Dl} }
\label{app:E}
In the following lemma we describe how to perform the projection over the set $\Dla$ in a finite number of steps.
Without loss of generality we consider the case that $\lambda=1$, the case regarding a general value of $\lambda$ 
immediately follows by a rescaling argument.
\begin{lemma}
Let $Q$ be a $d \times d$ symmetric matrix and let $U\Gamma U\trans$ be an {eigen}-decomposition of $Q$, with $\Gamma = {\rm Diag}(\gamma_1,\dots,\gamma_d)$. Then the solution of the problem
\begin{equation*} \label{lemma_proj}
{\hat D} = {\rm argmin} \Big \{ \||D-Q\||_2^2: D \succeq 0, {\tr}(D) \leq 1 \Big \}
\end{equation*}
is given by 
${\hat D} = {Q}$ if $Q$ satisfies the constraints and ${\hat D}=U \Theta U\trans$ otherwise, where $\Theta = {\rm Diag}(\theta_1,\dots,\theta_d)$, $\theta_i = \max(0, \gamma_i - a)$ for $i \in \{1,\dots,d\}$, and the nonnegative parameter $a$ is uniquely defined by the equation $\sum_{i=1}^d  \max(0, \gamma_i - a) = 1$.
\end{lemma}

The proof of the above lemma follows a standard path of reducing the matrix problem to a vector problem by application of von Neumann trace inequality, after which an argument based on Lagrange multipliers is employed, see e.g. \cite[Theorem 15]{Andrew}. Note that the explicit equation defining the parameter $a$ can be solved efficiently in $O(d \log d)$ time, hence the computational cost of the projection is dominated by the computational cost $O(d^3)$ 
of performing the eigen-decomposition of $Q$.

\end{document}